%% file: iclr2027_conference.tex
\documentclass{article} 
\usepackage{iclr2027_conference,times}

\renewcommand{\headrulewidth}{0pt}

\usepackage{hyperref}
\usepackage{url}
\usepackage{amsmath,amssymb,amsthm}

\usepackage{booktabs,multirow,graphicx,wrapfig}

\theoremstyle{plain}
\newtheorem{theorem}{Theorem}

\title{Bridging KV-Cache Quantization and Linear Attention: From Theory to Pretrained Weight Migration}

\author{%
{Kaicheng Xiao\thanks{Equal contribution.}},\hspace{0.5em}
{Liran Dong\footnotemark[1]},\hspace{0.5em}
Haotian Li,\hspace{0.5em}
Guoliang Xing\\
The Chinese University of Hong Kong\\
\texttt{\{xk023, dl123, lh023, glxing\}@ie.cuhk.edu.hk}
}

\iclrfinalcopy

\begin{document}

\maketitle

\begin{abstract}
KV-cache quantization and linear attention are two representative approaches to tackling the storage and computational costs of Transformers.
KV-cache quantization compresses individual KV entries into discrete codes but retains all entries, whereas linear attention recurrently aggregates multiple historical KV contributions into a fixed-size continuous state but can introduce interference.
This contrast raises the question of whether per-KV compression and multi-KV aggregation can be bridged within a single mechanism for efficient attention.
We identify RAM-Net as such a bridge through soft assignments over a discrete address space.
These assignments determine recurrent updates to the continuous slot state associated with each address.
Under a restricted RAM-Net construction, we prove that soft address assignments extend hard quantized matching to a separable read-write overlap that locally approximates full-attention similarity and supports recurrent aggregation.
These connections further enable Transformer-to-RAM-Net weight migration through a new path based on a soft-quantized intermediate construction.
Across nine pretrained Transformer models from 0.3B to 7B parameters, RAM-Net recovers an average of 87.1\% of the teachers' accuracy gains over random guessing across six commonsense and knowledge tasks using only a 500M-token budget per model.
\end{abstract}

\input{sec/1_intro}

\input{sec/2_bkg}

\input{sec/3_theory}

\input{sec/4_migration}

\input{sec/5_eval}

\input{sec/6_related}

\input{sec/7_discussion}

\bibliography{iclr2027_conference}
\bibliographystyle{iclr2027_conference}

\appendix

\input{sec/appendix_proof}

\end{document}

%% file: sec/1_intro.tex
\section{Introduction}

Transformer models have become the dominant architecture for language modeling~\citep{vaswani2017attention,brown2020language}. 
However, the memory use and per-query computation of full attention grow with context length because all historical keys and values are stored in the KV cache and accessed for each query.
Two representative approaches reduce these costs in different ways.
KV-cache quantization compresses individual KV entries into discrete codes~\citep{liu2024kivi,savkin2025nestquant}, but retains all entries, so memory still grows with context length.
Linear attention instead recurrently aggregates multiple historical KV contributions into a fixed-size continuous state~\citep{katharopoulos2020transformers,choromanski2021performer}, while this aggregation can cause interference among them.
This contrast raises the question of whether per-KV compression and multi-KV aggregation can be bridged within a single mechanism for efficient attention.
We identify RAM-Net~\citep{xiao2026ramnet} as such a bridge, combining a discrete address space with recurrent continuous slot states.
This bridge provides a new view of efficient attention and a new path for weight migration.

We theoretically explain how RAM-Net's soft address assignments connect discrete-code quantization, full attention, and linear attention within a restricted subset of RAM-Net.
These assignments generalize hard code selection in KV quantization, recovering quantized-key matching in the hard-write limit while reads remain soft.
We further prove that the overlap between soft read and write assignments locally approximates the exponential dot-product similarity of full attention without reconstructing keys from discrete codes.
Moreover, the address overlap is exactly separable into read and write terms. 
This separability enables historical KV contributions to be recurrently aggregated into continuous slot states, yielding the linear-style continuous state construction.

These theoretical connections further enable practical Transformer-to-RAM-Net migration across a broad range of pretrained Transformers.
Unlike prior Transformer-to-linear migration methods~\citep{goldstein2025radlads,bick2024transformers}, our migration introduces a soft-quantized intermediate attention model that replaces exponential dot-product matching with separable address overlap.
To construct the soft-quantized intermediate attention model, we first compress and align the Transformer's query and key projections to match its attention patterns.
We then fold a codebook into the aligned projections to implement soft quantization, while copying the remaining shared parameters from the Transformer.
Finally, we initialize RAM-Net with the intermediate model's parameters and progressively fine-tune it to compensate for approximation errors and architectural differences.

We evaluate the migration design on nine publicly available pretrained Transformer models covering five model families, multi-head attention (MHA) and grouped-query attention (GQA)~\citep{ainslie2023gqa} configurations, and scales from 0.3B to 7B parameters.
With a migration budget of only 500M tokens per model, RAM-Net recovers an average of 87.1\% of the teachers' accuracy gains over random guessing across six commonsense and knowledge tasks.
Ablations on the intermediate construction show that the theory-guided design improves attention alignment with the teacher and generally leads to lower output error during head-level recovery.
Together with the downstream gains from staged recovery under the same total token budget, these findings support the practical value of our theory-guided migration design.

%% file: sec/2_bkg.tex
\section{Background}
\label{sec:background}

In sequence modeling, attention uses learnable projections to map an input hidden state $\mathbf h_t$ to a query $\mathbf q_t$, a key $\mathbf k_t$, and a value $\mathbf v_t$.
Under the test-time training (TTT) view~\citep{sun2025learning}, attention constructs a query-to-value predictor $f_t$ from historical KV samples $\{(\mathbf k_s,\mathbf v_s)\}_{s\le t}$ and evaluates it on the current query as $\mathbf o_t=f_t(\mathbf q_t)$.
The predictor $f_t$ is therefore defined over the query-key (QK) matching space. 

For causal full attention~\citep{vaswani2017attention}, its predictor forms a weighted sum of historical values using softmax-normalized dot-product similarities between the query and historical keys:
\begin{equation}
f_t(\mathbf q)=\sum_{s\le t}\alpha_{t,s}(\mathbf q)\mathbf v_s,
\quad
\alpha_{t,s}(\mathbf q)=
\frac{\exp(\mathbf q^\top\mathbf k_s/\tau)}
{\sum_{j\le t}\exp(\mathbf q^\top\mathbf k_j/\tau)}.
\label{eq:full_attn}
\end{equation}

Here $\tau>0$ is the softmax temperature, and $\alpha_{t,s}(\mathbf q)$ are the softmax scores.
Since softmax assigns larger weights to keys with larger dot-product similarities, full attention can be viewed as a \textbf{soft nearest-neighbor (Soft-NN) predictor} in QK space. This Soft-NN predictor uses the complete set of historical KV samples, so its storage and per-query computation grow with context length. Fig.~\ref{fig:attention_atlas} illustrates how full attention is modified by KV-cache quantization and linear attention, and how RAM-Net connects these two directions, as detailed in the following three subsections.

\begin{figure}[t]
\centering
\includegraphics[width=0.75\linewidth]{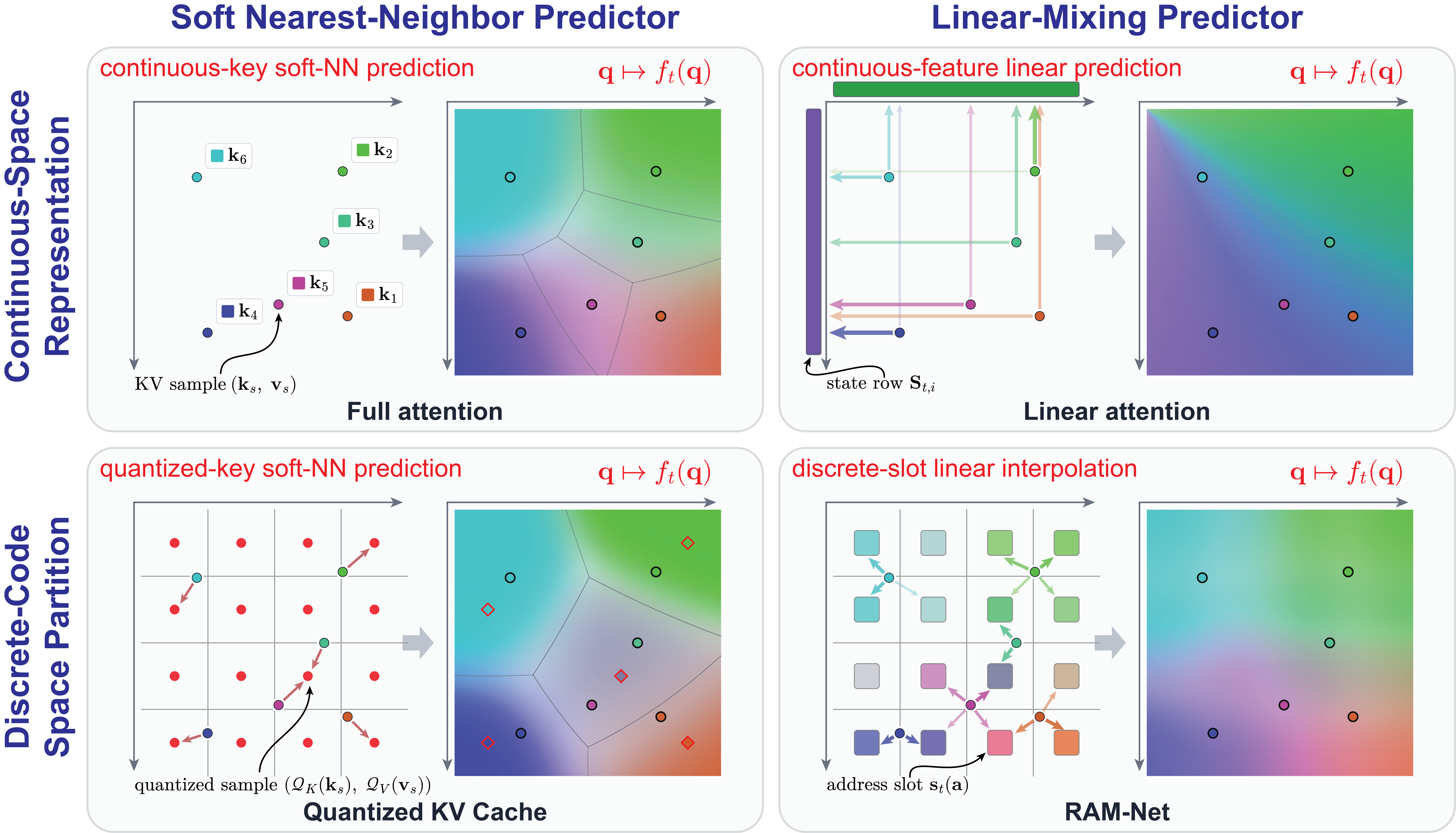}
\vspace{-0.3cm}
\caption{
\textbf{Comparison of attention predictors in QK space.}
For each sub-plot, the left shows how historical KV information is stored or aggregated, while the right visualizes the resulting $f_t(\mathbf q)$ over QK space. RAM-Net retains the discrete-code representation of KV-cache quantization and the continuous-state aggregation of linear attention.
}
\vspace{-0.5cm}
\label{fig:attention_atlas}
\end{figure}


\subsection{Quantized KV Caches: Discrete-Code Soft-NN Predictor}
\label{sec:bkg_vq}

KV-cache quantization compresses historical KV samples into a finite set of numerical levels or codewords. Quantizing keys partitions QK space into regions whose keys share the same reconstruction. For example, scalar quantization encodes individual coordinates~\citep{liu2024kivi,micikevicius2022fp8}, whereas vector quantization jointly encodes coordinate groups~\citep{savkin2025nestquant,son2025nsnquant}. Let $\mathcal{Q}_{\mathrm K}$ and $\mathcal{Q}_{\mathrm V}$ denote the quantize-dequantize operators for keys and values. This yields a \textbf{discrete-code Soft-NN predictor}:
\begin{equation}
f_t(\mathbf q)=
\sum_{s\le t}\alpha_{t,s}(\mathbf q)\mathcal{Q}_{\mathrm V}(\mathbf v_s),
\quad
\alpha_{t,s}(\mathbf q)=
\frac{\exp\!\left(\mathbf q^\top \mathcal{Q}_{\mathrm K}(\mathbf k_s)/\tau\right)}
{\sum_{j\le t}\exp\!\left(\mathbf q^\top \mathcal{Q}_{\mathrm K}(\mathbf k_j)/\tau\right)}.
\label{eq:quant_attn}
\end{equation}
This predictor still stores all historical KV samples, but quantization reduces the storage required for each sample. The resulting matching quality depends on how accurately the reconstructed keys preserve the original similarities between query and keys.


\subsection{Linear Attention: Linear Predictor with Continuous Features }
\label{sec:bkg_linear}

Kernelized linear attention replaces the exponential dot-product similarity in Eq.~\eqref{eq:full_attn} with a separable query-key similarity given by $\phi(\mathbf q)\phi(\mathbf k)^\top$~\citep{katharopoulos2020transformers,choromanski2021performer}. This separable form allows historical KV contributions to be recurrently aggregated into a fixed-size and continuous-feature state $\mathbf S_t$. It yields a \textbf{linear predictor}:
\begin{equation}
f_t(\mathbf q)\propto\phi(\mathbf q)\mathbf S_t,
\quad
\mathbf S_t=\Lambda_t\mathbf S_{t-1}+\phi(\mathbf k_t)^\top\mathbf v_t.
\label{eq:bg_linear}
\end{equation}
Here, $\Lambda_t$ controls how the previous state is retained or modified~\citep{yang2024gla,gu2024mamba,yang2025gdn}. Maintaining this fixed-size state makes storage and per-query computation independent of context length. However, the fixed state size limits the predictor's capacity for fine-grained matching to historical keys in the QK space.




\subsection{RAM-Net: Discrete-address Linear-Interpolation Predictor}
\label{sec:bkg_ramnet}

RAM-Net~\citep{xiao2026ramnet} consists of a finite set of discrete addresses $\mathbf a\in\mathcal A$, each indexing a slot that stores a continuous state $\mathbf s_t(\mathbf a)$. 
Keys and queries produce soft write and read assignments over discrete addresses, respectively, determining which slots are updated and read. The discrete addresses act as codes, and the assignments softly partition QK space. Each slot stores a continuous state that recurrently aggregates historical KV contributions as in linear attention.
Specifically, $\mathbf k_t$ and $\mathbf q_t$ are each split into $U$ groups of address logits, and Product Softmax combines their group-wise softmax probabilities into write and read distributions over addresses. Only the Top-$K$ probabilities are retained for efficiency, which yields sparse write weights $w_t(\mathbf a)$ and read weights $r_t(\mathbf a)$. Gated Sparse Update (GSU) then recurrently updates the selected slots with $\mathbf v_t$. The predictor and GSU rule are given by


\begin{equation}
f_t(\mathbf q_t)\propto\sum_{\mathbf a\in\mathcal A}r_t(\mathbf a)\mathbf s_t(\mathbf a),
\quad
\begin{cases}
    m_t(\mathbf a) = \sigma\!\left( \operatorname{logit}(1-w_t(\mathbf a))-\gamma_t \right)
    \cdot m_{t-1}(\mathbf a)+w_t(\mathbf a) \\
    \mathbf s_t(\mathbf a) = \mathbf s_{t-1}(\mathbf a) + \frac{w_t(\mathbf a)}{m_t(\mathbf a)+\epsilon}
    \left(\mathbf v_t-\mathbf s_{t-1}(\mathbf a)\right)
\end{cases}.
\label{eq:bg_ram}
\end{equation}

Here, $f_t(\mathbf q_t)$ is defined by the address overlap of write and read weights, $\gamma_t$ is an input-dependent scalar that modulates the forgetting tendency across all slots, $m_t(\mathbf a)$ is a per-slot scalar that tracks the slot's inertia against new writes, and $\epsilon$ stabilizes the denominator.
The read weights linearly combine the continuous states stored at discrete addresses to produce the output, defining a \textbf{finite discrete-address linear-interpolation predictor}. RAM-Net therefore combines the discrete-code representation of KV quantization with the recurrent aggregation of linear attention. Sec.~\ref{sec:theory} establishes this bridge theoretically, showing how a separable address kernel locally approximates Soft-NN similarity and how GSU extends cumulative averaging with forgetting while retaining linear aggregation.

%% file: sec/3_theory.tex
\section{Bridging Discrete-Code Soft-NN and Continuous Linear Predictors}
\label{sec:theory}

This section establishes how the address space of RAM-Net connects discrete-code Soft-NN predictors and continuous-feature linear predictors, as illustrated in Fig.~\ref{fig:proof}.
Sec.~\ref{sec:circular-addresses} first introduces soft assignment distributions over a finite address space as a \textbf{discrete-code representation} of queries and keys and adopts a restricted circular construction for theoretical analysis.
Sec.~\ref{sec:address-matching} then shows that the overlap between soft read and write distributions locally approximates exponential dot-product similarity, enabling discrete-code Soft-NN matching in the address space without explicit reconstruction.
The address overlap is also separable into read and write terms.
Building on this separability, Sec.~\ref{sec:recurrent-readout} shows how a \textbf{linear predictor} aggregates historical KV contributions into recurrent slot states and how GSU extends the cumulative updates with retention.

\begin{figure}
    \centering
    \includegraphics[width=0.85\linewidth]{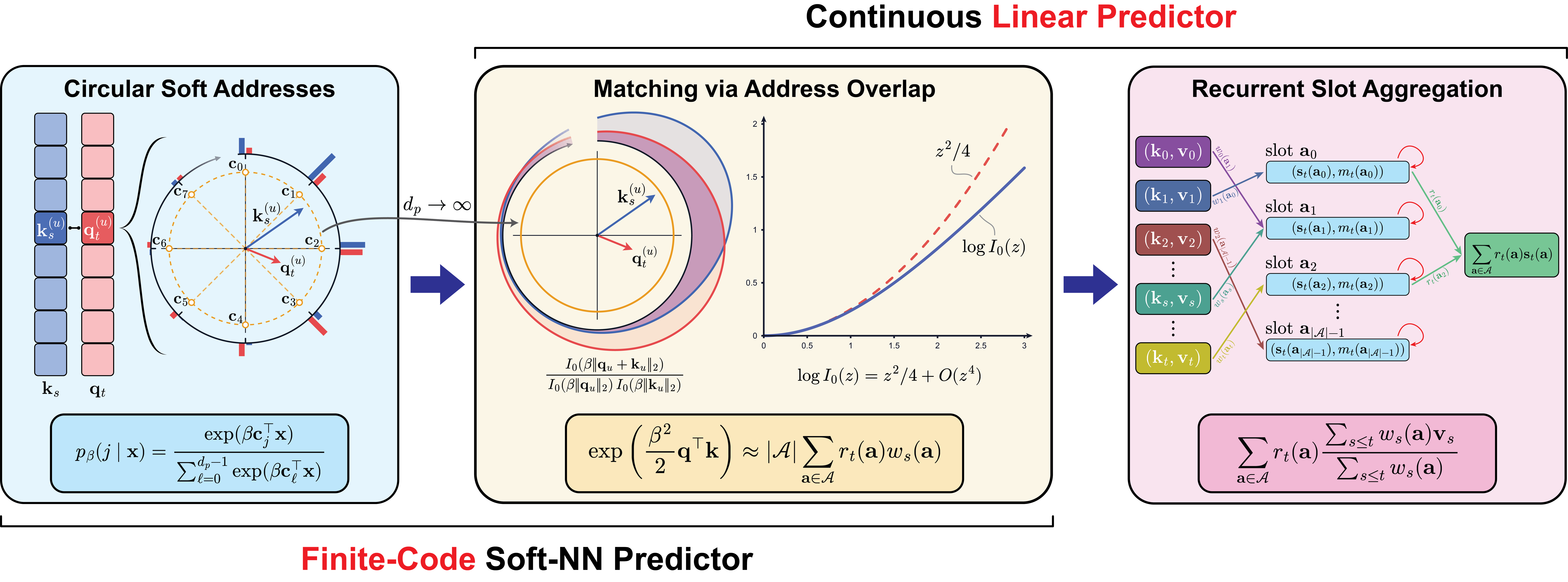}
    \vspace{-0.3cm}
    \caption{
    \textbf{Overview of Sec.~\ref{sec:theory}.}
    RAM-Net connects discrete-code Soft-NN and linear predictors through the separable overlap between read and write distributions built on circular soft addressing, which locally approximates Soft-NN matching.
    }
    \vspace{-0.5cm}
    \label{fig:proof}
\end{figure}

\subsection{Discrete-Code Representation via circular soft addressing}
\label{sec:circular-addresses}

\paragraph{Address space.}
In RAM-Net, we define the address space as $\mathcal A=\{0,\ldots,d_p-1\}^{U}$, where each address $\mathbf a=(a^{(0)},\ldots,a^{(U-1)})$ consists of $U$ factors, and each factor $a^{(u)}$ takes one of $d_p$ possible code indices. 
Each address $\mathbf a\in\mathcal A$ indexes a slot.
The read and write projections $\mathbf W_r$ and $\mathbf W_w$ each map the hidden state $\mathbf h_t$ to $U$ groups of logits, one for each factor.
Product Softmax normalizes each group and multiplies the resulting factor-wise probabilities to form a distribution over $\mathcal A$.

\paragraph{Restricted projection.}
For the $u$-th factor, let $\mathbf W_r^{(u)}$ and $\mathbf W_w^{(u)}$ denote the read and write projection blocks that produce the corresponding logits in Product Softmax.
We restrict these projection blocks to the forms $\mathbf W_r^{(u)}=\beta\mathbf C^\top\widetilde{\mathbf W}_Q^{(u)}$ and $\mathbf W_w^{(u)}=\beta\mathbf C^\top\widetilde{\mathbf W}_K^{(u)}$, respectively.
Here, $\mathbf C\in \mathbb{R}^{2 \times d_p}$ is a shared and fixed codebook, and the scalar $\beta\ge0$ controls the concentration of the code assignments.
This restricted parameterization defines a special case of RAM-Net in which each read and write projection block has rank at most two.
This parameterization introduces two-dimensional intermediate query and key vectors, with $\widetilde{\mathbf W}_Q^{(u)},\widetilde{\mathbf W}_K^{(u)}\in\mathbb R^{2\times d_{\mathrm{model}}}$ and $\mathbf q_t^{(u)},\mathbf k_s^{(u)}\in\mathbb R^2$.
$\widetilde{\mathbf W}_Q^{(u)}$ and $\widetilde{\mathbf W}_K^{(u)}$ map the hidden states to $\mathbf q_t^{(u)}=\widetilde{\mathbf W}_Q^{(u)}\mathbf h_t$ and $\mathbf k_s^{(u)}=\widetilde{\mathbf W}_K^{(u)}\mathbf h_s$.
This two-dimensional structure also supports an approximate correspondence between RoPE rotations and RAM-Net's cyclic address shifts (CAPE), as described in Appendix~\ref{app:cape-rope} and used for weight migration in Sec.~\ref{sec:migration}.

\paragraph{Circular code with soft assignments.}
For our theoretical construction, we choose $\mathbf C$ as a circular codebook with $d_p$ codes uniformly spaced on the unit circle. Each intermediate vector $\mathbf q_t^{(u)},\mathbf k_s^{(u)}\in\mathbb R^2$ is scored against these shared codes. 
For any $\mathbf x\in\mathbb R^2$, the code vectors and the corresponding soft assignment are defined as
\begin{equation}
\begin{aligned}
\mathbf c_j&=\left(\cos\frac{2\pi j}{d_p},\sin\frac{2\pi j}{d_p}\right)^\top,
\quad
p_\beta(j\mid\mathbf x)=\frac{\exp(\beta\mathbf c_j^\top\mathbf x)}{\sum_{\ell=0}^{d_p-1}\exp(\beta\mathbf c_\ell^\top\mathbf x)},
\quad
j=0,\ldots,d_p-1.
\end{aligned}
\label{eq:circular-assignment}
\end{equation}
The circular codebook is therefore $\mathbf C=\begin{bmatrix}\mathbf c_0 & \cdots & \mathbf c_{d_p-1}\end{bmatrix}$, with each scaled codeword $\beta\mathbf c_j$ having norm $\beta$.
Product Softmax combines the assignments from the $U$ factors to define the read and write distributions over $\mathcal A$ as $r_t(\mathbf a)=\prod_{u=0}^{U-1}p_\beta(a^{(u)}\mid\mathbf q_t^{(u)})$ and $w_s(\mathbf a)=\prod_{u=0}^{U-1}p_\beta(a^{(u)}\mid\mathbf k_s^{(u)})$.
As $\beta\to\infty$, each factor-wise assignment concentrates all probability mass on its highest-scoring code.
Product Softmax therefore concentrates on the address formed by these factor-wise maximizing code indices.
Thus, hard code assignment in discrete-code quantization is a limiting case of our soft address construction, as shown in Appendix~\ref{app:hard-addr}.

\subsection{Approximation to Soft-NN Predictors via Address Overlap}
\label{sec:address-matching}

The circular soft addressing distribution in Sec.~\ref{sec:circular-addresses} provides a discrete-code representation of queries and keys.
We write $\mathbf q_t=(\mathbf q_t^{(0)},\ldots,\mathbf q_t^{(U-1)})$ and $\mathbf k_s=(\mathbf k_s^{(0)},\ldots,\mathbf k_s^{(U-1)})$ for the $2U$-dimensional representations.
We show how the address distribution overlap can approximate the exponential dot-product similarity $\exp\!\left(\mathbf q_t^\top\mathbf k_s/\tau\right)$ used by the Soft-NN predictors of full attention and KV-cache quantization.
The following theorem establishes this approximation in a local regime.

\begin{theorem}[Local approximation of Soft-NN matching]
\label{thm:circular-matching}
Consider the circular soft-addressing construction in Sec.~\ref{sec:circular-addresses}.
For fixed $U$ and $d_p\ge4$, in the local regime $\max_u\beta(\|\mathbf q_t^{(u)}\|_2+\|\mathbf k_s^{(u)}\|_2)\to0$, the address overlap satisfies
\begin{equation}
\exp\!\left(\frac{\beta^2}{2}\mathbf q_t^\top\mathbf k_s\right)
\approx\kappa_{\mathcal A}(\mathbf q_t,\mathbf k_s)
=|\mathcal A|\sum_{\mathbf a\in\mathcal A}r_t(\mathbf a)w_s(\mathbf a).
\label{eq:address-overlap}
\end{equation}
\end{theorem}

\begin{proof}[Proof sketch]
Define $Z_{d_p}(\mathbf x)=\frac{1}{d_p}\sum_{j=0}^{d_p-1}\exp(\mathbf c_j^\top\mathbf x)$.
Using the factorized Product Softmax distributions defined in Sec.~\ref{sec:circular-addresses}, the address overlap can be written as

\begin{equation}
\kappa_{\mathcal A}(\mathbf q_t,\mathbf k_s)
=\prod_{u=0}^{U-1}\left[d_p\sum_{j=0}^{d_p-1}p_\beta(j\mid\mathbf q_t^{(u)})p_\beta(j\mid\mathbf k_s^{(u)})\right]
=\prod_{u=0}^{U-1}\frac{Z_{d_p}\!\left(\beta(\mathbf q_t^{(u)}+\mathbf k_s^{(u)})\right)}{Z_{d_p}(\beta\mathbf q_t^{(u)})Z_{d_p}(\beta\mathbf k_s^{(u)})}.
\label{eq:partition-overlap}
\end{equation}

As $d_p\to\infty$, the uniformly spaced circular codes approach a continuous unit circle, giving $Z_{d_p}(\mathbf x)\to I_0(\|\mathbf x\|_2)$, where $I_0$ is the modified Bessel function of the first kind of order zero.
In the local regime, substituting $\log I_0(z)=z^2/4+O(z^4)$ into the logarithm of Eq.~\eqref{eq:partition-overlap} in this limit and expanding the squared norms gives

\begin{equation}
\log\kappa_\infty(\mathbf q_t,\mathbf k_s)
=\lim_{d_p\to\infty}\log\kappa_{\mathcal A}(\mathbf q_t,\mathbf k_s)
=\frac{\beta^2}{2}\mathbf q_t^\top\mathbf k_s
+O\!\left(\beta^4\sum_{u=0}^{U-1}\bigl(\|\mathbf q_t^{(u)}\|_2+\|\mathbf k_s^{(u)}\|_2\bigr)^4\right).
\label{eq:bessel-matching}
\end{equation}

Appendix~\ref{app:softnn-approx} provides explicit error bounds and shows that, the finite-codebook correction is absorbed by the remainder in Eq.~\eqref{eq:bessel-matching}.
The same expansion therefore holds for $\log\kappa_{\mathcal A}(\mathbf q_t,\mathbf k_s)$.
Exponentiating yields Eq.~\eqref{eq:address-overlap}.
\end{proof}

The scaling factor $|\mathcal A|$ is introduced for the kernel comparison in the proof and does not affect the normalized attention weights.
Intuitively, even after rescaling, the address overlap remains bounded by $|\mathcal A|$, whereas the exponential dot-product kernel is unbounded, motivating a local approximation.
Matching the local expansion to the target kernel gives $\beta^2/2=1/\tau$ and hence $\beta^\star=\sqrt{2/\tau}$.
For $\tau=\sqrt{d_{\mathrm{head}}}$, this becomes $\beta^\star=\sqrt{2}\,d_{\mathrm{head}}^{-1/4}$.
We use $\beta^\star$ to guide weight migration in Sec.~\ref{sec:migration}.
Together, the local approximation in Theorem~\ref{thm:circular-matching} and the separability of the address overlap establish RAM-Net's connections to continuous Soft-NN, discrete-code Soft-NN, and linear predictors.
The local approximation allows RAM-Net to preserve Soft-NN matching in the discrete-code address space without the key reconstruction used in KV-cache quantization.
The address overlap is exactly separable into read and write terms, enabling the recurrent linear aggregation developed in Sec.~\ref{sec:recurrent-readout}.



\subsection{Linear Predictor via Slot Aggregation}
\label{sec:recurrent-readout}

Using the local approximation in Theorem~\ref{thm:circular-matching} over historical keys and the separable feature form for linear predictors provided by the address overlap in Eq.~\eqref{eq:address-overlap}, we aggregate historical KV contributions by address, giving

\begin{equation}
\frac{\sum_{s\le t}\exp\!\left(\beta^2/2\cdot\mathbf q_t^\top\mathbf k_s\right)\mathbf v_s}{\sum_{s\le t}\exp\!\left(\beta^2/2\cdot\mathbf q_t^\top\mathbf k_s\right)}
\approx
\frac{\sum_{s\le t}\kappa_{\mathcal A}(\mathbf q_t,\mathbf k_s)\mathbf v_s}{\sum_{s\le t}\kappa_{\mathcal A}(\mathbf q_t,\mathbf k_s)}
=
\frac{\sum_{\mathbf a}r_t(\mathbf a)m_t(\mathbf a)\mathbf s_t(\mathbf a)}{\sum_{\mathbf a}r_t(\mathbf a)m_t(\mathbf a)}.
\label{eq:mass-aware-readout}
\end{equation}

Here, $m_t(\mathbf a)=\sum_{s\le t}w_s(\mathbf a)$ is the cumulative write mass of slot $\mathbf a$, and $\mathbf s_t(\mathbf a)=\frac{\sum_{s\le t}w_s(\mathbf a)\mathbf v_s}{m_t(\mathbf a)}$ is its cumulative weighted mean of values, detail in Appendix~\ref{app:recurrence}.
The rightmost expression defines a mass-aware readout whose read coefficients are proportional to $r_t(\mathbf a)m_t(\mathbf a)$, so the contribution of each slot depends on both its query read score and accumulated write mass.

The GSU rule in Eq.~\eqref{eq:bg_ram} extends this cumulative form by introducing $\gamma_t$ to control the retention of historical contributions.
With $m_0(\mathbf a)=0$, $\mathbf s_0(\mathbf a)=\mathbf 0$, and $\epsilon=0$, taking $\gamma_t\to-\infty$ at every step removes this forgetting effect and reduces GSU to cumulative averaging:

\begin{equation}
m_t(\mathbf a)=m_{t-1}(\mathbf a)+w_t(\mathbf a),\qquad \mathbf s_t(\mathbf a)=\frac{m_{t-1}(\mathbf a)\mathbf s_{t-1}(\mathbf a)+w_t(\mathbf a)\mathbf v_t}{m_t(\mathbf a)}.
\label{eq:cumulative-slot-update}
\end{equation}

For finite $\gamma_t$, GSU retains its forgetting mechanism, where $m_t$ and $\mathbf s_t$ track the retained write mass and its corresponding weighted value mean.
The resulting predictor remains linear in the read-address features, while GSU updates the slot states online with controlled retention.

Original RAM-Net reads these slot states using $\sum_{\mathbf a}r_t(\mathbf a)\mathbf s_t(\mathbf a)$, omitting the write-mass reweighting.
The mass-aware readout ablation in Appendix~\ref{app:ramnet-plus}(Tables~\ref{tab:main-eval-mass} and~\ref{tab:mass-read-js}) shows that the mass-aware readout generally matches the Transformer's attention more closely before head-level recovery, but tends to yield slightly higher perplexity after migration than the original readout.
It may be related to the coupling of read weights to forgetting dynamics and a potential mismatch between mass reweighting and Top-$K$ selection based only on $r_t(\mathbf a)$.

%% file: sec/4_migration.tex
\section{Transformer-to-RAM-Net Weight Migration}
\label{sec:migration}

The connections established in Sec.~\ref{sec:theory} provide a basis for migrating pretrained Transformer weights to RAM-Net.
The migration proceeds in two stages.
First, we adapt the teacher's query-key matching to soft-quantized intermediate attention model, aiming to preserve its token-to-token attention patterns.
Second, we enable RAM-Net's recurrent updates and progressively fine-tune the model to reduce the remaining output discrepancies.


\subsection{Migration to Soft-Quantized Intermediate Attention}
\label{sec:migration-init}

Soft-quantized intermediate attention uses the overlap of circular read and write address assignments in place of exponential dot-product similarity (Secs.~\ref{sec:circular-addresses} and~\ref{sec:address-matching}), allowing us to align token-to-token matching before GSU.
We construct it by performing low-rank and positional alignment of the teacher's query and key projections, then folding them into read and write projections.

\paragraph{Low-rank and positional alignment.}
The circular construction in Sec.~\ref{sec:circular-addresses} uses $U$ two-dimensional factors.
Each two-dimensional factor can serve as a RoPE rotation plane, with rotations at frequencies $\omega_u=2\pi/d_p^{u+1}$ for $u=0,\ldots,U-1$ approximately corresponding to CAPE's cyclic address shifts (Appendix~\ref{app:cape-rope}).
Since the address count $d_p^U$ grows exponentially with $U$, we compress each teacher head's query and key vectors from their original dimension $d_k$ to $2U$, motivated by evidence that these representations admit low-rank approximations~\citep{saxena2024eigen}.
We introduce adapters $\mathbf P_Q,\mathbf P_K\in\mathbb R^{2U\times d_k}$ to implement this reduction, aiming to preserve position-dependent token matching.
For each head, we consider both RoPE with the target frequencies $\omega_u$ and no positional encoding (NoPE).
Using the Fourier representation of rotary matching~\citep{hua2025fourier}, we obtain closed-form adapter initializations from calibration statistics through frequency projection and low-rank factorization (Appendix~\ref{app:low-rank-construction}).
After refining the RoPE and NoPE candidates against the teacher's attention maps using forward Kullback-Leibler divergence (KL), we retain RAM-Net's original CAPE pattern across layers and assign heads to CAPE or NoPE based on their relative alignment errors.

\paragraph{Circular soft quantization.}
After selecting the adapters and positional modes, we map the aligned query and key representations to soft circular address assignments.
Let $\widetilde{\mathbf C}_U=\operatorname{Diag}(\mathbf C,\ldots,\mathbf C)$ denote the block-diagonal expansion of the circular codebook with $U$ copies.
Given the teacher projections $\mathbf W_Q^t$ and $\mathbf W_K^t$, we combine low-rank alignment, the circular construction, and temperature matching to obtain the student projections $\mathbf W_r^s=\beta^\star\widetilde{\mathbf C}_U^\top\mathbf P_Q\mathbf W_Q^t$ and $\mathbf W_w^s=\beta^\star\widetilde{\mathbf C}_U^\top\mathbf P_K\mathbf W_K^t$,  where $\beta^\star$ is the assignment concentration defined in Sec.~\ref{sec:address-matching} with $\tau=\sqrt{2U}$.
Product Softmax then yields address assignments whose overlap defines the intermediate attention similarity.

\subsection{RAM-Net Conversion and Progressive Recovery}
\label{sec:migration-recovery}

Recovery proceeds at the head, attention-block, and full-model levels. 
We convert the intermediate model in Sec.~\ref{sec:migration-init} to RAM-Net and optimize its RAM-Net-specific parameters to compensate for output differences introduced by Top-$K$ slot selection and GSU.
Notably, we optimize the read and write projection matrices $\mathbf W^s_r$ and $\mathbf W^s_w$ without the low-rank and shared-codebook constraints used in Sec.~\ref{sec:migration-init}. 
For \textbf{head-level recovery}, we adapt address routing and GSU updates to match the teacher's per-head outputs. For \textbf{attention-block recovery}, we train the attention branch to match the teacher's residual updates. Both stages use normalized mean squared error (NMSE), computed on head outputs before output projection and on attention-branch outputs before residual addition, respectively. For \textbf{full-model recovery}, we unfreeze and fine-tune all student parameters using next-token cross-entropy as the main objective. Appendix~\ref{sec:migration-add} provides the complete training details.

%% file: sec/5_eval.tex
\section{Experiments}
\label{sec:experiments}

\subsection{Experimental Setup}
\label{sec:exp-setup}

We evaluate Transformer-to-RAM-Net weight migration on pretrained models ranging from 340M to 7B parameters. 
We use nine publicly available pretrained teacher models, whose architectures are summarized in Table~\ref{tab:teacher-architectures}. 
Each RAM-Net matches its teacher's architecture and tokenizer, except for the attention mechanism.
Unless otherwise specified, we use $d_p=4$, $U=5$, and Top-$K=32$.
For GQA~\citep{ainslie2023gqa}, each KV head is repeated over its query-head group to form the equivalent multi-head attention (MHA) weights used for migration.
All migration experiments use the same total budget of 500M tokens, with per-stage token allocations, learning rates, and other training details provided in Table~\ref{tab:migration-overview} and Appendix ~\ref{sec:migration-add}.
The tokens come from the 10BT sample of the FineWeb-Edu dataset~\cite{huggingfacefw_2024}. 
Full-model recovery uses two NVIDIA H200 GPUs with FSDP2, while all other stages use a single GPU. 

\subsection{Language Model Evaluation Benchmarks}
\label{sec:exp-main}

Table~\ref{tab:main-eval-no-mass} reports language modeling performance and accuracy on six common-knowledge tasks.
Across the nine teachers, the Wiki PPL gap between RAM-Net and its teacher ranges from 1.22 to 5.09, while the average-accuracy gap across the six tasks ranges from 1.8 to 4.5 percentage points.
We use the relative score $R=\tfrac{s-r}{t-r}$ to avoid overstating retention when both models perform near chance; $s$, $t$, and $r$ denote RAM-Net, teacher, and chance accuracies.
Across teachers, $R$ ranges from $79.2\%$ to $95.0\%$, with a mean of $87.1\%$.
These results show that the migration procedure retains most of the teachers' performance on these benchmarks across attention types, model families, and scales.
We also report the mass-aware readout variant (Sec.~\ref{sec:recurrent-readout}) in Appendix~\ref{app:ramnet-plus}, which shows similar overall trends. 
Table~\ref{tab:longbench} and Table~\ref{tab:longbench-full} in the appendix further show that RAM-Net partially recovers teacher performance on code completion and summarization tasks from LongBench.

\input{figs/eval/table_main}

\subsection{Ablation Studies}
\label{sec:exp-ablation}

\noindent\textbf{Low-rank alignment and circular codebook.}
We ablate the low-rank alignment and circular codebook in Sec.~\ref{sec:migration-init}, measuring the Jensen--Shannon divergence (JS) between teacher and RAM-Net attention distributions after intermediate attention construction and before head-level recovery.
The \emph{r/w projection} and \emph{kq projection adapter} omit learned low-rank alignment, learning projections from scratch and transferring teacher QK weights through random adapters, respectively.
The codebook baselines retain learned low-rank alignment to isolate the effects of codebook design.
The \emph{random codebook} tests the circular construction against shared random codes.
Inspired by FAVOR+~\citep{choromanski2021performer}, the \emph{orthogonal random codebook} replaces circular codes with randomly scaled two-dimensional directions obtained by QR decomposition of Gaussian matrices.
We also compare learned codebooks~\citep{alonso2026online,vandendool2026avq} with the fixed circular construction: \emph{shared $k$-means} learns one shared $d_p$-centroid codebook from calibration using query and key across factors, while \emph{per-factor PQ} learns one per factor~\citep{jegou2011product}.
Table~\ref{tab:address-initialization-ablation} shows that the circular codebook yields lower JS than all alternatives across the nine teachers.
In the appendix, Fig.~\ref{fig:nmse_app} further shows that the circular codebook generally yields lower head NMSE during recovery;
Fig.~\ref{fig:sweep_app} and Table~\ref{tab:order-harness} evaluate the effects of varying $d_p$ and $U$.


\begin{wrapfigure}{R}{0.4\textwidth}
  \centering
  \vspace{-0.5cm}
  \includegraphics[width=\linewidth]{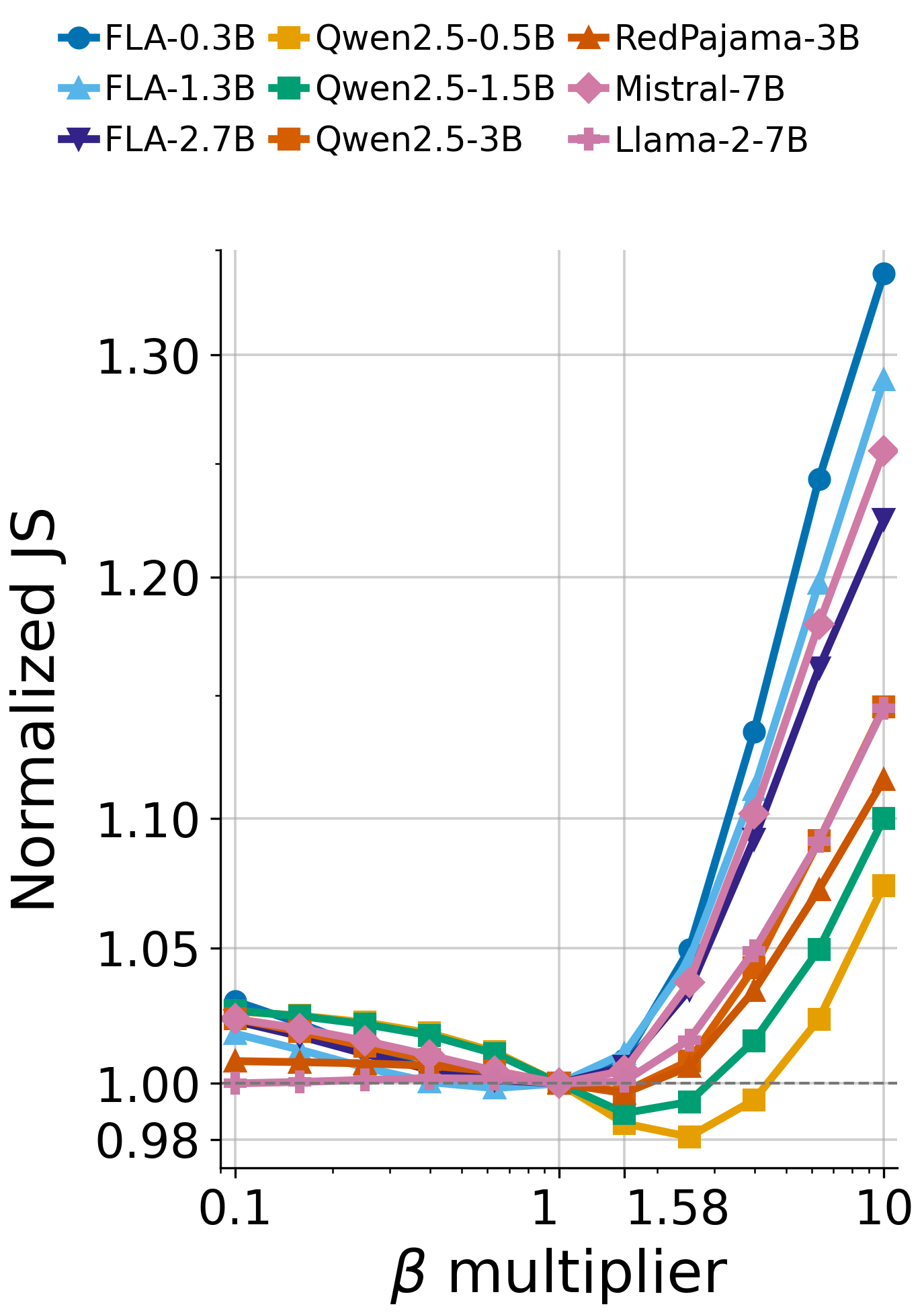}
  \vspace{-0.8cm}
  \caption{Attn. JS under different circular assignment concentrations $\beta$.}
  \vspace{-0.5cm}
  \label{fig:beta-js}
\end{wrapfigure}

\noindent\textbf{Code assignment concentration $\beta$.}
The concentration $\beta$ controls how the address overlap approximates exponential dot-product similarity. 
We sweep $\beta/\beta^\star$, where $\beta^\star$ is derived in Sec.~\ref{sec:address-matching}.
For each teacher, all sweep configurations use the same low-rank adapters.
We report normalized JS (Fig.~\ref{fig:beta-js}) and normalized head NMSE (Fig.~\ref{fig:beta-nmse} in the appendix) before head-level recovery, each normalized by its value at $\beta^\star$ for the same teacher.
Across teachers, $\beta^\star$ yields near-minimal JS and head NMSE, although some minima occur around $1.58\beta^\star$.
For well-aligned query-key pairs, higher-order terms can reduce the overlap relative to the local approximation, so a slightly larger $\beta$ may partly compensate by sharpening code assignments.

\noindent\textbf{Migration pipeline ablation.}
We assess the value of staged recovery in Sec.~\ref{sec:migration} on Qwen2.5-1.5B under a fixed 500M-token budget (Table~\ref{tab:stage-ablation}).
The \emph{complete pipeline} achieves lower Wiki PPL and higher downstream accuracy than \emph{full-model only}, which allocates the entire budget to full-model recovery.
Removing either \emph{head-level recovery} or \emph{attention-block recovery} lowers downstream accuracy, even when its tokens are reallocated to full-model recovery.
To assess the preset CAPE pattern, \emph{pattern-free CAPE selection} determines CAPE head counts solely from relative RoPE alignment errors.
It slightly lowers logit KL and Wiki PPL but also lowers downstream accuracy relative to the preset pattern.
Appendix Table~\ref{tab:stage-ablation-details} and Fig.~\ref{fig:stage_error} provide intermediate results and stage-wise errors, respectively.


\input{figs/eval/table_js}
\input{figs/eval/table_twin}

%% file: figs/eval/table_main.tex
\begin{table*}[t]
  \centering
    \caption{Zero-shot language modeling and commonsense reasoning performance of Teacher-to-RAM-Net migration.}
  \label{tab:main-eval-no-mass}
  \setlength{\tabcolsep}{2.4pt}
  \renewcommand{\arraystretch}{1.08}
  \resizebox{\textwidth}{!}{%
  \begin{tabular}{@{}lllccccccccccc@{}}
    \toprule
    Family & Teacher & Model
      & Tokens
      & Ctx.
      & Wiki PPL $\downarrow$
      & ARC-E $\uparrow$
      & ARC-C $\uparrow$
      & PIQA $\uparrow$
      & SciQ $\uparrow$
      & COPA $\uparrow$
      & Hella. $\uparrow$
      & $\operatorname{Avg.}$ $\uparrow$
      & $\operatorname{Avg.} R$ $\uparrow$ \\
    \midrule

    \multirow{6}{*}{FLA}
      & \multirow{2}{*}{\shortstack[l]{\texttt{transformer-340M-}\\\texttt{10B}}}
      & Teacher & 10B & 4K & 30.05 & 57.4 & 24.1 & 65.6 & 82.5 & 73.0 & 32.1 & 55.8
      & \multirow{2}{*}{86.5} \\
      & & RAM-Net & 0.5B & 2K
      & 35.14 & 54.7 & 23.7 & 64.7 & 80.2 & 65.0 & 31.0 & 53.2 & \\
    \cmidrule(lr){2-14}
      & \multirow{2}{*}{\shortstack[l]{\texttt{transformer-1.3B-}\\\texttt{100B}}}
      & Teacher & 100B & 2K
      & 17.66 & 56.1 & 23.9 & 70.1 & 85.3 & 75.0 & 38.5 & 58.1
      & \multirow{2}{*}{86.4} \\
      & & RAM-Net & 0.5B & 2K
      & 21.93 & 52.7 & 24.8 & 68.3 & 78.8 & 71.0 & 35.6 & 55.2 & \\
    \cmidrule(lr){2-14}
      & \multirow{2}{*}{\shortstack[l]{\texttt{transformer-2.7B-}\\\texttt{100B}}}
      & Teacher & 100B & 2K
      & 15.23 & 59.4 & 26.5 & 71.9 & 87.9 & 75.0 & 42.0 & 60.4
      & \multirow{2}{*}{89.2} \\
      & & RAM-Net & 0.5B & 2K
      & 19.82 & 56.5 & 24.6 & 70.3 & 79.9 & 74.0 & 38.3 & 57.3 & \\

    \midrule
    \multirow{6}{*}{Qwen2.5}
      & \multirow{2}{*}{\texttt{0.5B-Base}}
      & Teacher & -- & 32K & 18.46 & 64.4 & 29.3 & 70.6 & 93.1 & 74.0 & 40.7 & 62.0
      & \multirow{2}{*}{79.2} \\
      & & RAM-Net & 0.5B & 4K
      & 23.35 & 63.6 & 26.4 & 68.5 & 81.9 & 73.0 & 36.9 & 58.4 & \\
    \cmidrule(lr){2-14}
      & \multirow{2}{*}{\texttt{1.5B-Base}}
      & Teacher & -- & 32K & 12.60 & 75.2 & 41.1 & 75.6 & 94.2 & 83.0 & 50.2 & 69.9
      & \multirow{2}{*}{95.0} \\
      & & RAM-Net & 0.5B & 4K
      & 16.33 & 75.3 & 41.5 & 74.3 & 89.8 & 82.0 & 45.6 & 68.1 & \\
    \cmidrule(lr){2-14}
      & \multirow{2}{*}{\texttt{3B-Base}}
      & Teacher & -- & 32K & 12.41 & 77.6 & 45.0 & 78.6 & 96.1 & 86.0 & 55.0 & 73.0
      & \multirow{2}{*}{91.5} \\
      & & RAM-Net & 0.5B & 4K
      & 13.63 & 77.2 & 43.7 & 77.0 & 92.9 & 79.0 & 50.7 & 70.1 & \\

    \midrule
    \multirow{2}{*}{RedPajama}
      & \multirow{2}{*}{\texttt{INCITE-Base-3B-v1}}
      & Teacher & 800B & 2K
      & 11.76 & 68.0 & 31.7 & 74.8 & 90.4 & 77.0 & 47.6 & 64.9
      & \multirow{2}{*}{79.7} \\
      & & RAM-Net & 0.5B & 2K
      & 16.65 & 63.6 & 27.7 & 72.2 & 80.6 & 77.0 & 41.9 & 60.5 & \\

    \midrule
    \multirow{2}{*}{Mistral}
      & \multirow{2}{*}{\texttt{7B-v0.1}}
      & Teacher & -- & 8K
      & 8.22 & 80.3 & 50.9 & 80.9 & 96.4 & 92.0 & 61.5 & 77.0
      & \multirow{2}{*}{89.4} \\
      & & RAM-Net & 0.5B & 4K
      & 10.14 & 77.2 & 44.1 & 79.7 & 91.5 & 87.0 & 58.0 & 72.9 & \\

    \midrule
    \multirow{2}{*}{LLaMA}
      & \multirow{2}{*}{\texttt{2-7B-hf}}
      & Teacher & 2T & 4K & 8.76 & 75.3 & 42.7 & 77.7 & 94.8 & 88.0 & 57.1 & 72.6
      & \multirow{2}{*}{87.2} \\
      & & RAM-Net & 0.5B & 4K
      & 11.84 & 71.6 & 39.0 & 76.2 & 89.2 & 81.0 & 51.6 & 68.1 & \\

    \bottomrule
  \end{tabular}%
  }

  \vspace{2pt}
  \parbox{\textwidth}{\scriptsize
  Avg. is the arithmetic mean of the six task accuracies. Relative score $\operatorname{Avg.} R$ gives $\operatorname{Avg.}((s-r)/(t-r))$ over tasks for which $t>r$ with 95\% confidence.}
  \vspace{-0.8cm}
\end{table*}

%% file: figs/eval/table_js.tex
\begin{table*}[t]
  \centering
  \caption{Attention JS divergence of intermediate attention construction.}
  \label{tab:address-initialization-ablation}
  \scriptsize
  \setlength{\tabcolsep}{4.0pt}
  \renewcommand{\arraystretch}{1.0}
  \resizebox{\textwidth}{!}{%
  \begin{tabular}{@{}lccccccccc@{}}
    \toprule
    Method
      & FLA-0.3B
      & FLA-1.3B
      & FLA-2.7B
      & Qwen-0.5B
      & Qwen-1.5B
      & Qwen-3B
      & Red.-3B
      & Mist.-7B
      & LLaMA-7B \\
    \midrule

    r/w projection
      & 0.561
      & 0.563
      & 0.568
      & 0.625
      & 0.617
      & 0.601
      & 0.611
      & 0.619
      & 0.629 \\
    QK projection adapter
      & 0.570
      & 0.558
      & 0.571
      & 0.627
      & 0.629
      & 0.623
      & 0.601
      & 0.615
      & 0.615 \\
    Random codebook
      & 0.441
      & 0.440
      & 0.437
      & 0.486
      & 0.476
      & 0.498
      & 0.443
      & 0.430
      & 0.389 \\
    Orth. Rand. codebook
      & 0.487
      & 0.515
      & 0.525
      & 0.574
      & 0.566
      & 0.560
      & 0.563
      & 0.568
      & 0.586 \\
    Shared $k$-means codebook
      & 0.415
      & 0.412
      & 0.419
      & 0.466
      & 0.453
      & 0.481
      & 0.421
      & 0.411
      & 0.377 \\
    Per-factor PQ codebook
      & 0.412
      & 0.423
      & 0.431
      & 0.487
      & 0.476
      & 0.487
      & 0.446
      & 0.445
      & 0.407 \\
    \textbf{Circular codebook}
      & \textbf{0.398}
      & \textbf{0.394}
      & \textbf{0.402}
      & \textbf{0.452}
      & \textbf{0.444}
      & \textbf{0.470}
      & \textbf{0.412}
      & \textbf{0.389}
      & \textbf{0.359} \\

    \bottomrule
  \end{tabular}
  }
  \vspace{-0.3cm}
\end{table*}

%% file: figs/eval/table_twin.tex
\begin{table*}[t]
  \centering
  \scriptsize
  \begin{minipage}[t]{0.485\textwidth}
    \vspace{0pt}
    \centering
    \caption{LongBench results.}
    \label{tab:longbench}
    \setlength{\tabcolsep}{2.2pt}
    \renewcommand{\arraystretch}{1.0}
    \begin{tabular*}{\linewidth}{@{\extracolsep{\fill}}llrrrr@{}}
      \toprule
      Model & Method & LCC & MultiN. & SAMSum & Avg. $\uparrow$ \\
      \midrule
      \multirow{2}{*}{FLA-1.3B}
        & Teacher & 48.3 & 6.6 & 6.2 & 20.3 \\
        & RAM-Net & 36.4 & 7.2 & 9.3 & 17.6 \\
      \midrule
      \multirow{2}{*}{Qwen-3B}
        & Teacher & 14.1 & 24.5 & 45.5 & 28.0 \\
        & RAM-Net & 15.5 & 15.9 & 29.5 & 20.3 \\
      \midrule
      \multirow{2}{*}{Mistral-7B}
        & Teacher & 70.0 & 15.5 & 28.2 & 37.9 \\
        & RAM-Net & 33.4 & 12.1 & 22.1 & 22.5 \\
      \bottomrule
    \end{tabular*}
  \end{minipage}
  \hfill
  \begin{minipage}[t]{0.485\textwidth}
    \vspace{0pt}
    \centering
    \caption{Pipeline ablations on Qwen2.5-1.5B.}
    \label{tab:stage-ablation}
    \setlength{\tabcolsep}{3pt}
    \renewcommand{\arraystretch}{1.08}
    \begin{tabular*}{\linewidth}{@{\extracolsep{\fill}}lrrrr@{}}
      \toprule
      Method & Logit KL $\downarrow$ & Wiki PPL $\downarrow$ & Avg. $\uparrow$ & R $\uparrow$ \\
      \midrule
      Teacher & -- & 12.60 & 69.9 & -- \\
      \midrule
      \textbf{Complete} & 0.164 & 16.33 & \textbf{68.1} & \textbf{95.0} \\
      Pattern-free CAPE select. & \textbf{0.162} & \textbf{16.17} & 67.6 & 93.7 \\
      w/o Head Rec.  & 0.166 & 16.31 & 66.9 & 90.8 \\
      w/o Attn.-Block Rec. & 0.181 & 16.64 & 66.9 & 90.2 \\
      Full-Model only & 0.197 & 16.82 & 66.6 & 89.7 \\
      \bottomrule
    \end{tabular*}
  \end{minipage}
  \vspace{-0.2cm}
\end{table*}

%% file: sec/6_related.tex
\section{Related Work}

\subsection{Linear-Time Attention via Quantization and Kernel Approximation}

Vector-quantized attention reduces attention cost by replacing individual keys with a finite codebook. Transformer-VQ~\citep{lingle2024transformervq} uses quantized keys with recurrent per-code statistics, whose connection to our framework through hard writes and soft reads is derived in Appendix~\ref{app:hard-addr}. AVQ-Attention~\citep{vandendool2026avq} adaptively refines the codebook according to attention importance, while online VQ attention~\citep{alonso2026online} uses sparse online updates to maintain a larger fixed-size memory.
Kernelized attention instead uses separable features to approximate softmax matching. Performer/FAVOR+~\citep{choromanski2021performer} approximates the exponential dot-product kernel with positive orthogonal random features. Our construction connects these two views by using discrete-code soft assignments whose separable read and write overlap locally approximates exponential dot-product similarity and supports recurrent aggregation.

\subsection{Transformer-to-Linear Weight Migration}

Prior work has migrated pretrained Transformers to various linear or recurrent architectures, including recurrent attention in SUPRA~\citep{mercat2024linearizing}, kernelized linear attention in Hedgehog~\citep{zhang2024hedgehog} and LoLCATs~\citep{zhang2025lolcats}, Mamba models in MOHAWK~\citep{bick2024transformers} and Llamba~\citep{bick2025llamba}, RWKV variants in RADLADS~\citep{goldstein2025radlads}, gated recurrent models in Liger~\citep{lan2025liger}, and linearization for Gated DeltaNet~\citep{kuzina2026key}. These methods generally convert directly to the target architectures. Attention-to-Mamba~\citep{moudgil2026attention} instead uses kernelized linear attention as an intermediate before conversion to Mamba. Unlike these Transformer-to-linear methods, our migration proposes a distinct path by introducing an intermediate construction derived from quantized Transformers.

%% file: sec/7_discussion.tex
\section{Conclusion and Limitations}

In this paper, we establish RAM-Net as a theoretical bridge between discrete-code quantization and linear attention, and use this connection to enable efficient Transformer-to-RAM-Net weight migration. One limitation is that our derivation is local and assumes a fixed circular codebook with projection blocks of rank at most two, while practical RAM-Net uses unrestricted projections. Our migration experiments are also conducted under a fixed token budget without further tuning. We leave these limitations for future work.

%% file: sec/appendix_proof.tex
\section{Additional Proofs and Analysis}
\label{app:proofs-and-analysis}

\subsection{Full Proof and Extensions of Theorem~\ref{thm:circular-matching}}
\label{app:softnn-approx}

We provide a complete and generalized proof of Theorem~\ref{thm:circular-matching}, allowing separate read and write assignment concentrations (i.e., scaled codeword norms) and deriving an explicit approximation error bound stated in Theorem~\ref{thm:circular-matching-full}.

\begin{theorem}[Error bound with separate read and write concentrations and finite codebooks]
\label{thm:circular-matching-full}
Under the circular-codebook and factorized read and write setup of Theorem~\ref{thm:circular-matching}, let $\beta_r,\beta_w\ge0$ denote the read and write codeword norms.
With $\|\mathbf c_j\|_2=1$, the corresponding codewords are $\beta_r\mathbf c_j$ and $\beta_w\mathbf c_j$, respectively.
Let $M=\max_{0\le u<U}\max\{\beta_r\|\mathbf q_t^{(u)}\|_2,\beta_w\|\mathbf k_s^{(u)}\|_2\}$ and $\xi_{d_p,M}=2e^M M^{d_p}/d_p!$.
The log-overlap satisfies
\begin{equation}
        \left|\log\kappa_{\mathcal A}(\mathbf q_t,\mathbf k_s)-\frac{\beta_r\beta_w}{2}\mathbf q_t^\top\mathbf k_s\right|
        \le U\left[\frac{M^4}{4}+\frac{3\xi_{d_p,M}}{1-\xi_{d_p,M}}\right]
    ,\qquad \xi_{d_p,M}<1.
     \label{eq:circular-matching-bound}
\end{equation}
\end{theorem}

\begin{proof}
We first derive the continuous-codebook limit and its dot-product approximation, and then bound the dot-product approximation error and finite-codebook error separately.

For the uniform circular codewords $\mathbf c_j=(\cos(2\pi j/d_p),\sin(2\pi j/d_p))^\top$, introduce the partition function
\begin{equation*}
    Z_{d_p}(\mathbf x)=\frac{1}{d_p}\sum_{j=0}^{d_p-1}\exp(\mathbf c_j^\top\mathbf x).
\end{equation*}
Here the read and write scales are absorbed into the argument of $Z_{d_p}$.
Since $\mathcal A=\{0,\ldots,d_p-1\}^U$, the product structure of the read and write distributions allows the overlap to factorize over the $U$ factors:
\begin{equation}
\setlength{\jot}{1.5pt}
    \begin{aligned}
        \kappa_{\mathcal A}(\mathbf q_t,\mathbf k_s)
        &=|\mathcal A|\sum_{\mathbf a\in\mathcal A}r_t(\mathbf a)w_s(\mathbf a)\\
        &=d_p^U\sum_{\mathbf a\in\mathcal A}
        \prod_{u=0}^{U-1}\frac{\exp(\beta_r\mathbf c_{a^{(u)}}^\top\mathbf q_t^{(u)})}
        {\sum_{j=0}^{d_p-1}\exp(\beta_r\mathbf c_j^\top\mathbf q_t^{(u)})}\cdot
        \prod_{u=0}^{U-1}\frac{\exp(\beta_w\mathbf c_{a^{(u)}}^\top\mathbf k_s^{(u)})}
        {\sum_{j=0}^{d_p-1}\exp(\beta_w\mathbf c_j^\top\mathbf k_s^{(u)})}\\
        &=\prod_{u=0}^{U-1}\left[d_p\sum_{i=0}^{d_p-1}
        \frac{\exp(\beta_r\mathbf c_i^\top\mathbf q_t^{(u)})}
        {\sum_{j=0}^{d_p-1}\exp(\beta_r\mathbf c_j^\top\mathbf q_t^{(u)})}
        \frac{\exp(\beta_w\mathbf c_i^\top\mathbf k_s^{(u)})}
        {\sum_{j=0}^{d_p-1}\exp(\beta_w\mathbf c_j^\top\mathbf k_s^{(u)})}
        \right]\\
        &=\prod_{u=0}^{U-1}
        \frac{d_p\sum_{j=0}^{d_p-1}\exp\!\left(\mathbf c_j^\top(\beta_r\mathbf q_t^{(u)}+\beta_w\mathbf k_s^{(u)})\right)}
        {d_p Z_{d_p}(\beta_r\mathbf q_t^{(u)})\cdot d_p Z_{d_p}(\beta_w\mathbf k_s^{(u)})}\\
        &=\prod_{u=0}^{U-1}
        \frac{Z_{d_p}(\beta_r\mathbf q_t^{(u)}+\beta_w\mathbf k_s^{(u)})}
        {Z_{d_p}(\beta_r\mathbf q_t^{(u)})Z_{d_p}(\beta_w\mathbf k_s^{(u)})}.
    \end{aligned}
    \label{eq:circular-overlap-product}
\end{equation}

Let $Z_\infty(\mathbf x)=\lim_{d_p\to\infty}Z_{d_p}(\mathbf x)$ and $\kappa_\infty(\mathbf q_t,\mathbf k_s)=\lim_{d_p\to\infty}\kappa_{\mathcal A}(\mathbf q_t,\mathbf k_s)$. 
The circular Riemann sum converges to an integral; substituting this limit into Eq.~\eqref{eq:circular-overlap-product} gives
\begin{equation}
    \begin{aligned}
        Z_\infty(\mathbf x)
        &=\lim_{d_p\to\infty}\frac{1}{d_p}\sum_{j=0}^{d_p-1}
        \exp\!\left(\|\mathbf x\|_2 \cos\!\left(\tfrac{2\pi j}{d_p}-\arg\mathbf x\right) \right)\\
        &=\frac{1}{2\pi}\int_0^{2\pi}\exp(\|\mathbf x\|_2\cos\theta)\,d\theta\\
        &=I_0(\|\mathbf x\|_2)\\
        \kappa_\infty(\mathbf q_t,\mathbf k_s)
        &=\prod_{u=0}^{U-1}
        \frac{I_0\!\left(\|\beta_r\mathbf q_t^{(u)}+\beta_w\mathbf k_s^{(u)}\|_2\right)}
        {I_0\!\left(\|\beta_r\mathbf q_t^{(u)}\|_2\right)I_0\!\left(\|\beta_w\mathbf k_s^{(u)}\|_2\right)}.
    \end{aligned}
    \label{eq:circular-continuous-kernel}
\end{equation}
Here $I_0$ is the modified Bessel function of the first kind of order zero. 

To identify the leading term of the log-kernel, use $I_0(\rho)=1+\tfrac{\rho^2}{4}+O(\rho^4)$ and $\log(1+x)=x-\tfrac{x^2}{2}+O(x^3)$ as $x\to0$:
\begin{equation*}
    \log I_0(\rho)
    =\left(\tfrac{\rho^2}{4}+O(\rho^4)\right)-\tfrac{1}{2}\left(\tfrac{\rho^2}{4}+O(\rho^4)\right)^2+O(\rho^6)
    =\frac{\rho^2}{4}+O(\rho^4).
\end{equation*}
Substituting this expansion into the logarithm of Eq.~\eqref{eq:circular-continuous-kernel} as $M\to0$ and using the squared-norm identity $\|\beta_r\mathbf q_t^{(u)}+\beta_w\mathbf k_s^{(u)}\|_2^2=\beta_r^2\|\mathbf q_t^{(u)}\|_2^2+2\beta_r\beta_w{\mathbf q_t^{(u)}}^\top\mathbf k_s^{(u)}+\beta_w^2\|\mathbf k_s^{(u)}\|_2^2$ yields
\begin{equation}
    \begin{aligned}
        \log \kappa_\infty(\mathbf q_t,\mathbf k_s)
        &=\sum_{u=0}^{U-1}\Big[
            \log I_0\!\left(\|\beta_r\mathbf q_t^{(u)}+\beta_w\mathbf k_s^{(u)}\|_2\right)\\
            &\qquad\qquad
            -\log I_0\!\left(\|\beta_r\mathbf q_t^{(u)}\|_2\right)
            -\log I_0\!\left(\|\beta_w\mathbf k_s^{(u)}\|_2\right)
        \Big]\\
        &=  \frac{1}{4}\sum_{u=0}^{U-1}\left[
                \|\beta_r\mathbf q_t^{(u)}+\beta_w\mathbf k_s^{(u)}\|_2^2
                -\beta_r^2\|\mathbf q_t^{(u)}\|_2^2
                -\beta_w^2\|\mathbf k_s^{(u)}\|_2^2
            \right]\\
            &\qquad\qquad
            +O\!\left(\sum_{u=0}^{U-1}\bigl(\beta_r\|\mathbf q_t^{(u)}\|_2+\beta_w\|\mathbf k_s^{(u)}\|_2\bigr)^4\right)\\
        &=  \frac{\beta_r\beta_w}{2}\mathbf q_t^\top\mathbf k_s 
            +O\!\left(\bigl(\beta_r\|\mathbf q_t\|_2+\beta_w\|\mathbf k_s\|_2\bigr)^4\right)
    \end{aligned}
    \label{eq:dense-expand}
\end{equation}
The last step uses a conservative remainder bound in terms of the norms of the full query and key vectors by the triangle inequality.
Eq.~\eqref{eq:dense-expand} gives the continuous-codebook version of the dot-product approximation in Theorem~\ref{thm:circular-matching}, generalized to separate read and write assignment concentrations $\beta_r$ and $\beta_w$.

We now replace this asymptotic remainder with an explicit bound.
With $M$ as defined in Theorem~\ref{thm:circular-matching-full}, the triangle inequality gives $\|\beta_r\mathbf q_t^{(u)}+\beta_w\mathbf k_s^{(u)}\|_2\le2M$ for every $u$.
It therefore suffices to control the scalar approximation error of $\log I_0$ on these argument ranges.

For $\rho\ge0$, set $t=\rho^2/4$, giving
\begin{equation*}
    \left(1+\frac{t}{2}\right)^2
    =1+t+\frac{t^2}{4}
    \le I_0(\rho)
    =\sum_{m=0}^{\infty}\frac{t^m}{(m!)^2}
    \le\sum_{m=0}^{\infty}\frac{t^m}{m!}
    =e^t.
\end{equation*}
To obtain a lower bound for $\log I_0$, we use the following integral representation and the inequality $1-s\le1/(1+s)$ for $s\ge0$:
\begin{equation*}
(t/2)-\frac{(t/2)^2}{2}=\int_0^{\tfrac{t}{2}}(1-s)\,ds \le \int_0^{\tfrac{t}{2}}\frac{1}{1+s}\,ds = \log(1+\frac{t}{2}).
\end{equation*}
Taking logarithms of the power-series bounds and applying this inequality yields
\begin{equation*}
    t-\frac{t^2}{4}\le2\log\left(1+\frac{t}{2}\right)\le\log I_0(\rho)\le t.
\end{equation*}
Rearranging isolates the quadratic approximation error $0\le t - \log I_0(\rho) \le t^2/4$.
Define this scalar remainder by
\begin{equation*}
    \eta(\rho)=\frac{\rho^2}{4}-\log I_0(\rho)
\end{equation*}
and substitute $t=\rho^2/4$ to obtain
\begin{equation}
    0\le \eta(\rho)\le\frac{\rho^4}{64}.
     \label{eq:bessel-log-remainder}
\end{equation}

Substituting $\log I_0(\rho)=\frac{\rho^2}{4}-\eta(\rho)$ into the logarithm of Eq.~\eqref{eq:circular-continuous-kernel} gives the exact remainder:
\begin{equation*}
    \begin{aligned}
        \log\kappa_\infty(\mathbf q_t,\mathbf k_s)-\frac{\beta_r\beta_w}{2}\mathbf q_t^\top\mathbf k_s
        &=\sum_{u=0}^{U-1}\Big[
            \eta\!\left(\beta_r\|\mathbf q_t^{(u)}\|_2\right)+\eta\!\left(\beta_w\|\mathbf k_s^{(u)}\|_2\right)\\
        &\qquad\qquad-\eta\!\left(\|\beta_r\mathbf q_t^{(u)}+\beta_w\mathbf k_s^{(u)}\|_2\right)
        \Big].
    \end{aligned}
\end{equation*}
By Eq.~\eqref{eq:bessel-log-remainder}, the first two terms in each summand lie in $[0,M^4/64]$, while the subtracted term lies in $[0,(2M)^4/64]$.
Summing these bounds gives
\begin{equation}
    -\frac{UM^4}{4}
    \le\log\kappa_\infty(\mathbf q_t,\mathbf k_s)-\frac{\beta_r\beta_w}{2}\mathbf q_t^\top\mathbf k_s
    \le\frac{UM^4}{32}.
\end{equation}
Thus the continuous-kernel approximation error satisfies
\begin{equation}
    \left|\log\kappa_\infty(\mathbf q_t,\mathbf k_s)-\frac{\beta_r\beta_w}{2}\mathbf q_t^\top\mathbf k_s\right|\le\frac{UM^4}{4}.
     \label{eq:circular-continuous-error}
\end{equation}

Now consider the error between the finite and continuous log-kernels, $\left|\log\kappa_{\mathcal A}(\mathbf q_t,\mathbf k_s)-\log\kappa_\infty(\mathbf q_t,\mathbf k_s)\right|$.
Set $\rho=\|\mathbf x\|_2$, and let $I_n$ denote the modified Bessel function of the first kind of order $n$.
Applying the Fourier-Bessel expansion $\exp(\rho\cos\phi)=I_0(\rho)+2\sum_{n=1}^{\infty}I_n(\rho)\cos(n\phi)$ to $Z_{d_p}(\mathbf x)$ shows that only frequencies divisible by $d_p$ survive the discrete average:
\begin{equation*}
  \begin{aligned}
        Z_{d_p}(\mathbf x)
        &=
        I_0(\rho)+2\sum_{n=1}^{\infty}I_n(\rho)
        \left[\frac{1}{d_p}\sum_{j=0}^{d_p-1}\cos\!\left(n\left(\tfrac{2\pi j}{d_p}-\arg\mathbf x\right)\right)\right]\\
        &=
        Z_\infty(\mathbf x)+2\sum_{\ell=1}^{\infty}I_{\ell d_p}(\rho)\cos(\ell d_p\arg\mathbf x).
    \end{aligned}  
\end{equation*}
These additional modes account for the finite-codebook error.
Dividing by $Z_\infty(\mathbf x)=I_0(\rho)>0$ and using $|\cos\theta|\le1$ gives
\begin{equation*}
    \left|\frac{Z_{d_p}(\mathbf x)}{ Z_\infty(\mathbf x)}-1\right|
    \le
    \frac{2\sum_{\ell=1}^{\infty}I_{\ell d_p}(\rho)}{I_0(\rho)}.
\end{equation*}

To bound this tail explicitly, compare the power series of $I_n$ and use $(m+n)!\ge m!n!$:
\begin{equation*}
        I_n(\rho)
        =\left(\frac{\rho}{2}\right)^n\sum_{m=0}^{\infty}\frac{(\rho^2/4)^m}{m!(m+n)!}
        \le\frac{(\rho/2)^n}{n!}\sum_{m=0}^{\infty}\frac{(\rho^2/4)^m}{(m!)^2}
        =\frac{(\rho/2)^n}{n!}I_0(\rho).
\end{equation*}
Apply the resulting bound $I_n(\rho)/I_0(\rho)\le(\rho/2)^n/n!$, and use $(d_p+m)!\ge d_p!m!$:
\begin{equation*}
    \begin{aligned}
        \frac{2\sum_{\ell=1}^{\infty}I_{\ell d_p}(\rho)}{I_0(\rho)}
        &\le2\sum_{\ell=1}^{\infty}\frac{(\rho/2)^{\ell d_p}}{(\ell d_p)!}
        \le2\sum_{n=d_p}^{\infty}\frac{(\rho/2)^n}{n!}=2\left(\frac{\rho}{2}\right)^{d_p}\sum_{m=0}^{\infty}\frac{(\rho/2)^m}{(d_p+m)!}\\
        &\le\frac{2(\rho/2)^{d_p}}{d_p!}\sum_{m=0}^{\infty}\frac{(\rho/2)^m}{m!}
        =\frac{2e^{\rho/2}(\rho/2)^{d_p}}{d_p!}.
    \end{aligned}
\end{equation*}
Since every argument norm in the kernel factors is at most $2M$, combining the preceding bounds gives, uniformly for $\|\mathbf x\|_2\le2M$,
\begin{equation}
    \frac{2\sum_{\ell=1}^{\infty}I_{\ell d_p}(\rho)}{I_0(\rho)}\le \xi_{d_p,M},\qquad \xi_{d_p,M}=\frac{2e^M M^{d_p}}{d_p!}.
    \label{eq:circular-partition-error}
\end{equation}
To compare the log-kernels, we convert this relative error bound into a logarithmic one.
Assume $\xi_{d_p,M}<1$ and $\|\mathbf x\|_2\le2M$, and set $s=Z_{d_p}(\mathbf x)/I_0(\|\mathbf x\|_2)-1$.
Eq.~\eqref{eq:circular-partition-error} gives $|s|\le\xi_{d_p,M}$, and every $v$ between $0$ and $s$ satisfies $1+v\ge1-\xi_{d_p,M}>0$.
Consequently,
\begin{equation*}
    |\log(1+s)|
    =\left|\int_0^s\frac{1}{1+v}\,dv\right|
    \le\left|\int_0^s\frac{1}{1-\xi_{d_p,M}}\,dv\right|
    \le\frac{\xi_{d_p,M}}{1-\xi_{d_p,M}}.
\end{equation*}
Substituting the definition of $s$ yields, uniformly for $\|\mathbf x\|_2\le2M$,
\begin{equation}
    \left|\log\frac{Z_{d_p}(\mathbf x)}{ Z_\infty(\mathbf x)}\right|\le\frac{\xi_{d_p,M}}{1-\xi_{d_p,M}}.
    \label{eq:circular-log-partition-error}
\end{equation}
By Eqs.~\eqref{eq:circular-overlap-product} and~\eqref{eq:circular-continuous-kernel}, the logarithm of each factor of $\kappa_{\mathcal A}/\kappa_\infty$ contains one numerator correction and two denominator corrections, all with argument norms at most $2M$.
Applying Eq.~\eqref{eq:circular-log-partition-error} to these $3U$ terms and using the triangle inequality gives
\begin{equation}
    \left|\log\frac{\kappa_{\mathcal A}(\mathbf q_t,\mathbf k_s)}
    {\kappa_\infty(\mathbf q_t,\mathbf k_s)}\right|
    \le\frac{3U\xi_{d_p,M}}{1-\xi_{d_p,M}}.
    \label{eq:circular-finite-error}
\end{equation}
Combining the finite-codebook bound in Eq.~\eqref{eq:circular-finite-error} with the continuous-kernel approximation bound in Eq.~\eqref{eq:circular-continuous-error} proves Eq.~\eqref{eq:circular-matching-bound}:
\begin{equation*}
    \begin{aligned}
        &\left|\log\kappa_{\mathcal A}(\mathbf q_t,\mathbf k_s)-\frac{\beta_r\beta_w}{2}\mathbf q_t^\top\mathbf k_s\right|\\
        \le&
         \left|\log\kappa_{\mathcal A}(\mathbf q_t,\mathbf k_s)-\log\kappa_\infty(\mathbf q_t,\mathbf k_s)\right|
        +\left|\log\kappa_\infty(\mathbf q_t,\mathbf k_s)-\frac{\beta_r\beta_w}{2}\mathbf q_t^\top\mathbf k_s\right|\\
        \le&
        U\left[\frac{M^4}{4}+\frac{3\xi_{d_p,M}}{1-\xi_{d_p,M}}\right]
    \end{aligned}
\end{equation*}

For fixed $d_p$, $\xi_{d_p,M}=O(M^{d_p})$ as $M\to0$, so $\xi_{d_p,M}<1$ holds for sufficiently small $M$.
The total error is therefore $O\!\left(U(M^4+M^{d_p})\right)$, and hence $O(UM^4)$ whenever $d_p\ge4$.
For fixed $U$, setting $\beta_r=\beta_w=\beta$ and exponentiating gives
\begin{equation*}
    \kappa_{\mathcal A}(\mathbf q_t,\mathbf k_s)
    =\exp\!\left(\frac{\beta^2}{2}\mathbf q_t^\top\mathbf k_s\right)
    \left[1+O(UM^4)\right],
    \qquad M\to0,
\end{equation*}
which recovers Theorem~\ref{thm:circular-matching}.
\end{proof}

\subsection{From Softmax Attention to Mass-Aware Recurrent Readout}
\label{app:recurrence}

We approximate causal softmax attention by a mass-aware readout over address states.
When the local regime holds uniformly over $s\le t$, the first step follows from Theorem~\ref{thm:circular-matching-full}; the remaining equalities are exact.
Canceling $|\mathcal A|$ and exchanging the sums over tokens and addresses yields a readout whose address weights are proportional to the read probability multiplied by the accumulated write mass:
\begin{equation}
    \setlength{\jot}{1.5pt}
    \begin{aligned}
        \frac{\sum_{s\le t}\exp\!\left(\frac{\beta_r\beta_w}{2}\mathbf q_t^\top\mathbf k_s\right)\mathbf v_s}
            {\sum_{s\le t}\exp\!\left(\frac{\beta_r\beta_w}{2}\mathbf q_t^\top\mathbf k_s\right)}
        &\approx\frac{\sum_{s\le t}\kappa_{\mathcal A}(\mathbf q_t,\mathbf k_s)\mathbf v_s}
            {\sum_{s\le t}\kappa_{\mathcal A}(\mathbf q_t,\mathbf k_s)}\\
        &=\frac{|\mathcal A|\sum_{s\le t}\sum_{\mathbf a\in\mathcal A}r_t(\mathbf a)w_s(\mathbf a)\mathbf v_s}
            {|\mathcal A|\sum_{s\le t}\sum_{\mathbf b\in\mathcal A}r_t(\mathbf b)w_s(\mathbf b)}\\
        &=\frac{\sum_{\mathbf a\in\mathcal A}r_t(\mathbf a)\sum_{s\le t}w_s(\mathbf a)\mathbf v_s}
            {\sum_{\mathbf b\in\mathcal A}r_t(\mathbf b)\sum_{s\le t}w_s(\mathbf b)}\\
        &=\sum_{\mathbf a\in\mathcal A}\frac{r_t(\mathbf a)m_t(\mathbf a)}
            {\sum_{\mathbf b\in\mathcal A}r_t(\mathbf b)m_t(\mathbf b)}
            \frac{\sum_{s\le t}w_s(\mathbf a)\mathbf v_s}{m_t(\mathbf a)}\\
        &=\sum_{\mathbf a\in\mathcal A}\frac{r_t(\mathbf a)m_t(\mathbf a)}
            {\sum_{\mathbf b\in\mathcal A}r_t(\mathbf b)m_t(\mathbf b)}\mathbf s_t(\mathbf a).
    \end{aligned}
    \label{eq:app-mass-aware-readout}
\end{equation}

Here $m_t(\mathbf a)$ is the cumulative write mass, and $\mathbf s_t(\mathbf a)$ is the write-weighted mean of values at address $\mathbf a$:
\begin{equation*}
m_t(\mathbf a)=\sum_{s\le t}w_s(\mathbf a),
\qquad
\mathbf s_t(\mathbf a)
=
\frac{\sum_{s\le t}w_s(\mathbf a)\mathbf v_s}{m_t(\mathbf a)}.
\end{equation*}
Appendix~\ref{app:gsu} shows how GSU recovers these cumulative statistics in its cumulative-averaging limit.

\subsection{GSU Dynamics and the Cumulative-Averaging Limit}
\label{app:gsu}

We show that GSU recovers the cumulative write mass and weighted value mean derived in Appendix~\ref{app:recurrence}.
For fixed soft write weights $0<w_t(\mathbf a)<1$, taking $\gamma_t\to-\infty$ at every step makes the retention gate tend to one:
\begin{equation*}
    \sigma\!\left( \operatorname{logit}(1-w_t(\mathbf a))-\gamma_t \right)\to1
\end{equation*}
In this limit, the mass update retains all previous mass and adds the current write weight.
Starting from $m_0(\mathbf a)=0$ and $\mathbf s_0(\mathbf a)=\mathbf0$, the accumulated mass is therefore
\begin{equation*}
    m_t(\mathbf a)=m_{t-1}(\mathbf a)+w_t(\mathbf a)=\sum_{s\le t}w_s(\mathbf a).
\end{equation*}
We set $\epsilon=0$ in the state update.
Substituting the mass recurrence then shows by induction that the state is the cumulative write-weighted mean.
The fourth equality uses the induction hypothesis; the base case follows from the zero initialization:
\begin{align*}
    \mathbf s_t(\mathbf a)
    &=\mathbf s_{t-1}(\mathbf a)
        +\frac{w_t(\mathbf a)}{m_t(\mathbf a)}\left(\mathbf v_t-\mathbf s_{t-1}(\mathbf a)\right)\\
    &=\frac{\left(m_t(\mathbf a)-w_t(\mathbf a)\right)\mathbf s_{t-1}(\mathbf a)+w_t(\mathbf a)\mathbf v_t}
        {m_t(\mathbf a)}\\
    &=\frac{m_{t-1}(\mathbf a)\mathbf s_{t-1}(\mathbf a)+w_t(\mathbf a)\mathbf v_t}
        {m_t(\mathbf a)}\\
    &=\frac{\sum_{s<t}w_s(\mathbf a)\mathbf v_s+w_t(\mathbf a)\mathbf v_t}
        {m_t(\mathbf a)}\\
    &=\frac{\sum_{s\le t}w_s(\mathbf a)\mathbf v_s}
        {\sum_{s\le t}w_s(\mathbf a)}.
\end{align*}
Thus GSU recovers cumulative weighted averaging in this limit.
Combined with the mass-aware readout in Eq.~\eqref{eq:app-mass-aware-readout}, it exactly reproduces the normalized address-kernel attention.

\begin{figure}
    \centering
    \includegraphics[width=.85\linewidth]{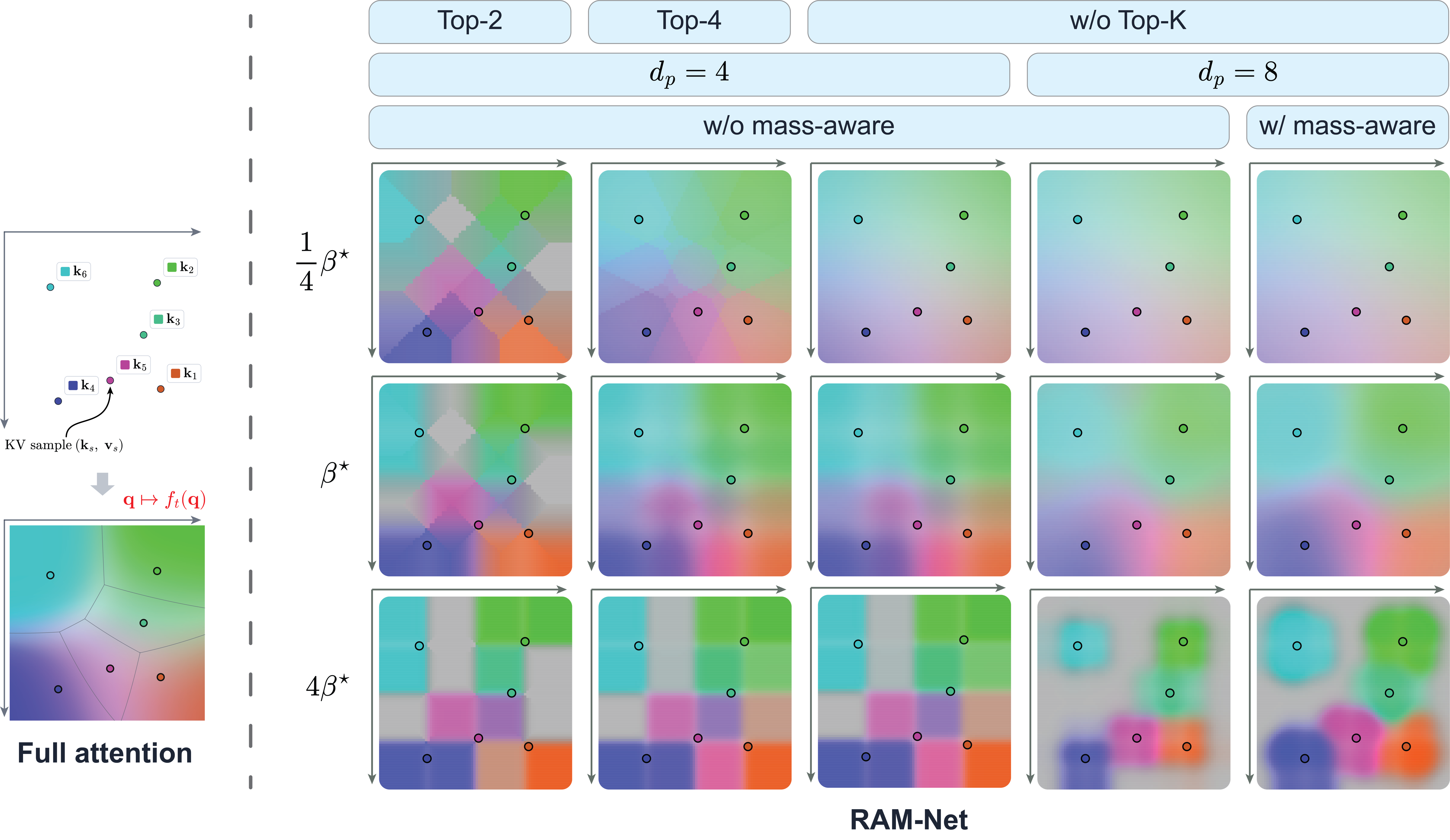}
    \caption{
    Value fields under different RAM-Net configurations.
    Keys and queries lie on $S^1\times S^1\subset\mathbb R^4$, with the plotted axes representing the angular coordinates of the two factors.
    Marker colors encode the stored values, while the background shows the predicted value $f_t(\mathbf q)$.
    Full attention provides the reference on the left.
    Rows vary the codeword norm around the reference setting $\beta^\star$, and columns vary Top-$K$, $d_p$, and mass-aware readout.
    }
    \label{fig:cmp-cfg}
\end{figure}

\subsection{Effects of Codebook Size, Sparsity, and Readout}
\label{app:cmp-cfg}

Building on the attention approximation in Appendix~\ref{app:softnn-approx}, we examine how codebook resolution, sparse address selection, and readout weighting affect the resulting value field.
Fig.~\ref{fig:cmp-cfg} compares these configurations with full attention, using a shared read and write codeword norm $\beta_r=\beta_w=\beta$.

The codeword norm $\beta$ controls how strongly reads and writes concentrate on nearby addresses.
Top-$K$ truncation introduces larger artifacts at small $\beta$, where diffuse address weights leave more probability mass outside the selected slots.
Increasing $d_p$ refines the angular grid and expands the address space to $d_p^U$ slots.
At moderate $\beta$, the finer codebook reduces coarse grid artifacts and produces smoother transitions that more closely resemble the full-attention reference.

Standard RAM-Net uses $\sum_{\mathbf a}r_t(\mathbf a)\mathbf s_t(\mathbf a)$, whereas the mass-aware readout in Eq.~\eqref{eq:app-mass-aware-readout} replaces the coefficients with $r_t(\mathbf a)m_t(\mathbf a)/\sum_{\mathbf b}r_t(\mathbf b)m_t(\mathbf b)$.
The two agree when the masses are equal and positive on the read support and $\sum_{\mathbf a}r_t(\mathbf a)=1$.
Otherwise, mass weighting changes the relative influence of slots and shifts the boundaries between value regions.
When these masses differ, mass weighting changes the relative influence of slots and can shift the boundaries between value regions.
Because the address masses depend jointly on the codebook factors, this weighting can introduce oblique transitions, as illustrated by the last two columns.

Although mass-aware readout recovers normalized address-kernel attention in the cumulative-averaging limit, its practical benefit may be limited by interactions with forgetting and sparse selection.
Appendix~\ref{app:ramnet-plus} compares the two readouts empirically.
First, forgetting can reduce a slot's mass relative to other slots, lowering its mass-aware readout weight even when the read scores remain unchanged.
This couples forgetting to both the state dynamics and the readout, potentially suppressing slots that remain relevant to the query.
Second, the efficient Top-$K$ decoder selects addresses using the factorized score, whereas the history-dependent mass generally breaks this factorization in $r_t(\mathbf a)m_t(\mathbf a)$.
Reweighting only the selected addresses can therefore miss slots that would rank highly under the mass-aware score.

\subsection{Hard-Write Limit and Vector-Quantized Attention}
\label{app:hard-addr}

We derive the hard-write limit from the exact address overlap in Eq.~\eqref{eq:circular-overlap-product}.
We keep the read codeword norm $\beta_r$ fixed and finite while taking the write codeword norm $\beta_w\to\infty$, yielding hard write assignments with soft reads.

For each key block, assume a unique maximizing code and define
\begin{equation*}
    j_s^{(u)}
    =\operatorname*{arg\,max}_{0\le j<d_p}\mathbf c_j^\top\mathbf k_s^{(u)}
    =\operatorname*{arg\,min}_{0\le j<d_p}\|\mathbf k_s^{(u)}-\mathbf c_j\|_2^2.
\end{equation*}
The second equality follows from $\|\mathbf c_j\|_2=1$, so the assignment is nearest-code vector quantization.
For nonzero key blocks with continuously distributed angles, ties occur with probability zero, so the maximizing code is unique with probability one.

Dividing the numerator and denominator of each factor-wise write probability by the largest exponential makes every nonmaximizing term vanish as $\beta_w\to\infty$.
Consequently,
\begin{equation*}
    \frac{\exp(\beta_w\mathbf c_j^\top\mathbf k_s^{(u)})}
    {\sum_{i=0}^{d_p-1}\exp(\beta_w\mathbf c_i^\top\mathbf k_s^{(u)})}
    \longrightarrow\mathbf 1_{\{j=j_s^{(u)}\}}.
\end{equation*}
Let $\mathbf a_s=(j_s^{(0)},\ldots,j_s^{(U-1)})$ be the selected product address; then $w_s(\mathbf a)\longrightarrow\mathbf 1_{\{\mathbf a=\mathbf a_s\}}$.
Define the reconstructed key by $\widehat{\mathbf k}_s=(\mathbf c_{j_s^{(0)}}^\top,\ldots,\mathbf c_{j_s^{(U-1)}}^\top)^\top$.
The hard write selects a single term in the overlap, giving
\begin{equation}
    \lim_{\beta_w\to\infty}\kappa_{\mathcal A}(\mathbf q_t,\mathbf k_s)
    =d_p^U\prod_{u=0}^{U-1}\frac{\exp(\beta_r\mathbf c_{j_s^{(u)}}^\top\mathbf q_t^{(u)})}
        {\sum_{j=0}^{d_p-1}\exp(\beta_r\mathbf c_j^\top\mathbf q_t^{(u)})}
    =\frac{\exp(\beta_r\mathbf q_t^\top\widehat{\mathbf k}_s)}
        {\prod_{u=0}^{U-1}Z_{d_p}(\beta_r\mathbf q_t^{(u)})}.
    \label{eq:app-hard-write-kernel}
\end{equation}
The denominator depends only on the query and therefore cancels when the attention weights are normalized over tokens:
\begin{equation}
    \lim_{\beta_w\to\infty}
    \frac{\sum_{s\le t}\kappa_{\mathcal A}(\mathbf q_t,\mathbf k_s)\mathbf v_s}
    {\sum_{s\le t}\kappa_{\mathcal A}(\mathbf q_t,\mathbf k_s)}
    =
    \frac{\sum_{s\le t}\exp(\beta_r\mathbf q_t^\top\widehat{\mathbf k}_s)\mathbf v_s}
    {\sum_{s\le t}\exp(\beta_r\mathbf q_t^\top\widehat{\mathbf k}_s)}.
    \label{eq:app-hard-write-attention}
\end{equation}

Eq.~\eqref{eq:app-hard-write-attention} gives softmax attention over product-quantized keys~\citep{jegou2011product}, with each of the $U$ blocks quantized using a fixed circular codebook of $d_p$ codewords.
Taking this hard-write limit after the cumulative-averaging limit in Appendix~\ref{app:gsu} makes each occupied slot's mass and state the count and mean of its assigned values.
With the mass-aware readout in Eq.~\eqref{eq:app-mass-aware-readout}, this recovers the per-code aggregation underlying Transformer-VQ~\citep{lingle2024transformervq}.

\subsection{Correspondence between RoPE and CAPE}
\label{app:cape-rope}

Consider the full soft address distributions of the circular construction in Sec.~\ref{sec:circular-addresses} before Top-$K$ truncation.
We flatten each address $\mathbf a\in\mathcal A$ into a slot index $i\in\{0,\ldots,d_p^U-1\}$, with $a^{(0)}$ as the least significant digit:
\begin{equation}
    i(\mathbf a)=\sum_{u=0}^{U-1}a^{(u)}d_p^u,
    \qquad
    a^{(u)}=\left\lfloor\frac{i}{d_p^u}\right\rfloor\bmod d_p,
    \qquad u=0,\ldots,U-1.
    \label{eq:cape-tuple-index}
\end{equation}
For each two-dimensional input $\mathbf x$, the circular assignment has the following log probability, using the partition function $Z_{d_p}$ defined in Appendix~\ref{app:softnn-approx}:
\begin{equation*}
    \begin{aligned}
        \log p_\beta(j\mid\mathbf x)
        &=\beta\mathbf c_j^\top\mathbf x-\log\!\left(d_pZ_{d_p}(\beta\mathbf x)\right)\\
        &=\beta\|\mathbf x\|_2\cos\!\left(\frac{2\pi}{d_p}j-\arg\mathbf x\right)
            -\log\!\left(d_pZ_{d_p}(\beta\mathbf x)\right).
    \end{aligned}
\end{equation*}
For a vector $\mathbf x=(\mathbf x^{(0)},\ldots,\mathbf x^{(U-1)})$ with two-dimensional blocks, we remove the normalization term, which is independent of $i$, to define the factor logit waveform:
\begin{equation}
    \Phi^{(u)}_{\mathbf x}(i)
    = \log p_\beta(a^{(u)}\mid\mathbf x^{(u)})+\log\!\left(d_pZ_{d_p}(\beta\mathbf x^{(u)})\right)
    =\beta\|\mathbf x^{(u)}\|_2\cos\!\left(\frac{2\pi}{d_p}a^{(u)}-\arg\mathbf x^{(u)}\right)
    \label{eq:cape-factor-logit}
\end{equation}
The read and write distributions are obtained by superposing their respective factor logit waves and applying a softmax over addresses:
\begin{equation*}
    r_t(\mathbf a)\propto\exp\!\left(\sum_{u=0}^{U-1}\Phi^{(u)}_{\mathbf q_t}(i(\mathbf a))\right)
    \qquad
    w_s(\mathbf a)\propto\exp\!\left(\sum_{u=0}^{U-1}\Phi^{(u)}_{\mathbf k_s}(i(\mathbf a))\right)
\end{equation*}

By Eq.~\eqref{eq:cape-tuple-index}, each factor waveform has period $d_p^{u+1}$ and is constant over blocks of $d_p^u$ slots.
It can be obtained by sampling and holding the following continuous wave, where $s\in\mathbb R$ denotes a continuous slot coordinate:
\begin{equation}
\widetilde\Phi^{(u)}_{\mathbf x}(s)=\beta\|\mathbf x^{(u)}\|_2\cos\!\left(\frac{2\pi}{d_p^{u+1}}s-\arg\mathbf x^{(u)}\right).
\label{eq:cape-continuous-wave}
\end{equation}
Evaluating this continuous wave at the left endpoint of each block of $d_p^u$ slots exactly recovers the discrete waveform in Eq.~\eqref{eq:cape-factor-logit}:
\begin{equation}
    \Phi^{(u)}_{\mathbf x}(i)
    =\widetilde\Phi^{(u)}_{\mathbf x}\!\left(
        d_p^u\left\lfloor\frac{i}{d_p^u}\right\rfloor
    \right),
    \qquad u=0,\ldots,U-1.
    \label{eq:cape-sample-hold}
\end{equation}

RoPE preserves each block's norm and rotates it by $\nu^{(u)}t$, where $\nu^{(u)}$ is the angular frequency of factor $u$.
Substituting this change into Eq.~\eqref{eq:cape-continuous-wave} gives
\begin{equation*}
    \begin{aligned}
        &\beta\|\mathbf x^{(u)}\|_2
        \cos\!\left(\frac{2\pi}{d_p^{u+1}}s-\arg\mathbf x^{(u)}-\nu^{(u)}t\right)\\
        &=\beta\|\mathbf x^{(u)}\|_2
        \cos\!\left[
            \frac{2\pi}{d_p^{u+1}}
            \left(s-\frac{d_p^{u+1}\nu^{(u)}}{2\pi}t\right)
            -\arg\mathbf x^{(u)}
        \right]\\
        &=\widetilde\Phi^{(u)}_{\mathbf x}\!\left(
            s-\frac{d_p^{u+1}\nu^{(u)}}{2\pi}t
        \right).
    \end{aligned}
\end{equation*}

This corresponds to translation at speed $d_p^{u+1}\nu^{(u)}/(2\pi)$ slots per token.
Under the cyclic-shift convention $f(i)\mapsto f((i-t)\bmod d_p^U)$, setting the translation speed to one slot per token gives
\begin{equation}
    \nu^{(u)}=\frac{2\pi}{d_p^{u+1}},\qquad u=0,\ldots,U-1.
     \label{eq:cape-rope-frequency}
\end{equation}
With Eq.~\eqref{eq:cape-rope-frequency}, every continuous factor wave becomes $\widetilde\Phi_{\mathbf x}^{(u)}(s-t)$, so their superposition also translates by $t$ slots.
Since softmax commutes with cyclic slot permutations, CAPE equivalently shifts the sampled-and-held logits, whereas RoPE translates the continuous waves before sampling and holding.
By Eq.~\eqref{eq:cape-sample-hold}, the two operations give the following factor logits at slot $i$:
\begin{equation*}
    \begin{aligned}
        \text{CAPE:}\quad&
        \widetilde\Phi_{\mathbf x}^{(u)}\!\left(
            d_p^u\left\lfloor\frac{i-t}{d_p^u}\right\rfloor
        \right),\\
        \text{RoPE:}\quad&
        \widetilde\Phi_{\mathbf x}^{(u)}\!\left(
            d_p^u\left\lfloor\frac{i}{d_p^u}\right\rfloor-t
        \right).
    \end{aligned}
\end{equation*}
This difference in the order of operations introduces a discretization error into the resulting address distributions.

\subsection{Low-rank Adapter Construction}
\label{app:low-rank-construction}

We derive the initialization in Sec.~\ref{sec:migration-init} by matching the query-key interactions induced by the teacher projections.
Consider one attention head and omit layer and head indices. Under the column-vector convention used in the main text, the teacher projections are $\mathbf W_Q,\mathbf W_K\in\mathbb R^{d_k\times d_{\mathrm{model}}}$, and the target projections are $\widetilde{\mathbf W}_Q,\widetilde{\mathbf W}_K\in\mathbb R^{2U\times d_{\mathrm{model}}}$. 
Let $\Delta=i-j$ be the key position $i$ minus the query position $j$, and let $\mathbf R_T(\Delta)$ and $\mathbf R_S(\Delta)$ denote the relative teacher and target rotations.
The unscaled teacher logit is $\mathbf h_j^\top\mathbf B_T(\Delta)\mathbf h_i$, where
\begin{equation}
    \mathbf B_T(\Delta)=\mathbf W_Q^\top\mathbf R_T(\Delta)\mathbf W_K,
    \qquad
    \mathbf B_S(\Delta)=\widetilde{\mathbf W}_Q^\top\mathbf R_S(\Delta)\widetilde{\mathbf W}_K.
\label{eq:low-rank-operators}
\end{equation}
These bilinear operators describe the interactions that determine attention logits.
Matching them therefore targets the joint behavior of the query and key projections rather than approximating each projection independently.

Let $p(\Delta)$ be the empirical relative-position distribution weighted by teacher attention. 
Define the hidden-state second moments $\mathbf\Sigma_Q=\mathbb E_j[\mathbf h_j\mathbf h_j^\top]$ and $\mathbf\Sigma_K=\mathbb E_i[\mathbf h_i\mathbf h_i^\top]$ to weight operator errors.
Here, $\mathbf\Sigma_Q$ and $\mathbf\Sigma_K$ reflect the different positional sampling distributions of queries and keys under causal attention.
Approximating squared logit error with factorized query and key second moments that do not depend on $\Delta$ gives the objective
\begin{equation}
    \mathcal L=\mathbb E_{\Delta\sim p}\left\|
        \mathbf\Sigma_Q^{1/2}\bigl[\mathbf B_T(\Delta)-\mathbf B_S(\Delta)\bigr]\mathbf\Sigma_K^{1/2}
    \right\|_F^2.
\label{eq:low-rank-kernel-objective}
\end{equation}
Write $\widehat{\mathbf B}_T(\Delta)=\mathbf\Sigma_Q^{1/2}\mathbf B_T(\Delta)\mathbf\Sigma_K^{1/2}$ for the weighted teacher operator.

\noindent\textbf{CAPE initialization.}
Each teacher RoPE plane contributes cosine and sine terms to $\mathbf B_T(\Delta)$~\citep{su2021roformer}.
We approximate these positional interactions using the target frequencies $\nu^{(u)}$, $u=0,\ldots,U-1$, specified in Eq.~\eqref{eq:cape-rope-frequency}.
Collect their cosine and sine functions in $\boldsymbol\phi(\Delta)=(\cos(\nu^{(0)}\Delta),\sin(\nu^{(0)}\Delta),\ldots,\cos(\nu^{(U-1)}\Delta),\sin(\nu^{(U-1)}\Delta))^\top$.
We first relax the two-dimensional projection structure by allowing $2U$ unrestricted coefficient matrices $\mathbf X_a\in\mathbb R^{d_{\mathrm{model}}\times d_{\mathrm{model}}}$.
The relaxed objective is
\begin{equation*}
    \mathcal L=\mathbb E_{\Delta\sim p}\left\|
        \widehat{\mathbf B}_T(\Delta)-\sum_{a=0}^{2U-1}\phi_a(\Delta)\mathbf X_a
    \right\|_F^2.
\end{equation*}
Differentiating the least-squares objective gives
\begin{equation*}
    \frac{\partial \mathcal L}{\partial \mathbf X_a}
    =2\left(
        \sum_b\mathbb E_{\Delta\sim p}[\phi_a(\Delta)\phi_b(\Delta)]\mathbf X_b
        -\mathbb E_{\Delta\sim p}[\phi_a(\Delta)\widehat{\mathbf B}_T(\Delta)]
    \right).
\end{equation*}
Let $G_{ab}=\mathbb E_{\Delta\sim p}[\phi_a(\Delta)\phi_b(\Delta)]$ denote the entries of the Gram matrix $\mathbf G$.
A solution of the normal equations is
\begin{equation}
\mathbf X_a^\star
=\sum_{b=0}^{2U-1}(\mathbf G^\dagger)_{ab}\mathbb E_{\Delta\sim p}[\phi_b(\Delta)\widehat{\mathbf B}_T(\Delta)].
\label{eq:low-rank-fourier-projection}
\end{equation}
Here, $\mathbf G^\dagger$ is the Moore-Penrose pseudoinverse. 
This gives the minimum-norm solution of the unconstrained frequency projection~\citep{penrose1956best}, with $\mathbf G$ accounting for correlations among the target basis functions.
We next seek target projections $\widetilde{\mathbf W}_Q$ and $\widetilde{\mathbf W}_K$ whose induced weighted operator approximates the frequency expansion from Eq.~\eqref{eq:low-rank-fourier-projection}:
\begin{equation*}
    \mathbf\Sigma_Q^{1/2}
    \widetilde{\mathbf W}_Q^\top
    \mathbf R_S(\Delta)
    \widetilde{\mathbf W}_K
    \mathbf\Sigma_K^{1/2}
    \approx
    \sum_{a=0}^{2U-1}\phi_a(\Delta)\mathbf X_a^\star.
\end{equation*}
For each target frequency $\nu^{(u)}$, we approximate the corresponding cosine and sine coefficients using a single rotation plane.
Let the rows of $\widetilde{\mathbf W}_Q^{(u)}\mathbf\Sigma_Q^{1/2}$ and $\widetilde{\mathbf W}_K^{(u)}\mathbf\Sigma_K^{1/2}$ be $\mathbf a_{u,0}^\top,\mathbf a_{u,1}^\top$ and $\mathbf b_{u,0}^\top,\mathbf b_{u,1}^\top$, respectively.
Then
\begin{equation}
    \begin{aligned}
        &\mathbf\Sigma_Q^{1/2}
        \left(\widetilde{\mathbf W}_Q^{(u)}\right)^\top
        \mathbf R_S^{(u)}(\Delta)
        \widetilde{\mathbf W}_K^{(u)}
        \mathbf\Sigma_K^{1/2}\\
        &=
        \begin{bmatrix}
            \mathbf a_{u,0}^\top\\
            \mathbf a_{u,1}^\top
        \end{bmatrix}^\top
        \begin{bmatrix}
            \cos(\nu^{(u)}\Delta) & -\sin(\nu^{(u)}\Delta)\\
            \sin(\nu^{(u)}\Delta) & \cos(\nu^{(u)}\Delta)
        \end{bmatrix}
        \begin{bmatrix}
            \mathbf b_{u,0}^\top\\
            \mathbf b_{u,1}^\top
        \end{bmatrix}\\
        &\approx
        \cos(\nu^{(u)}\Delta)\mathbf X_{2u}^\star
        +\sin(\nu^{(u)}\Delta)\mathbf X_{2u+1}^\star.
    \end{aligned}
    \label{eq:cape-coefficient-match}
\end{equation}

Based on Eq.~\eqref{eq:cape-coefficient-match}, define $\mathbf a_u = \mathbf a_{u,0}+i\mathbf a_{u,1}$, $\mathbf b_u = \mathbf b_{u,0}+i\mathbf b_{u,1}$, and $\mathbf M_u=\mathbf X_{2u}^\star+i\mathbf X_{2u+1}^\star$, where $i$ is the imaginary unit:
\begin{equation*}
    \begin{aligned}
        \operatorname{Re}[e^{-i\nu_u\Delta} \mathbf a_u\mathbf b_u^H]
        &=
        \cos(\nu^{(u)}\Delta)(\mathbf a_{u,0}\mathbf b_{u,0}^\top+\mathbf a_{u,1}\mathbf b_{u,1}^\top) 
            + \sin(\nu^{(u)}\Delta)(\mathbf a_{u,1}\mathbf b_{u,0}^\top - \mathbf a_{u,0}\mathbf b_{u,1}^\top)
        \\
        \operatorname{Re}[e^{-i\nu_u\Delta}\mathbf M_u]
        &=
        \cos(\nu^{(u)}\Delta)\mathbf X_{2u}^\star + \sin(\nu^{(u)}\Delta) \mathbf X_{2u+1}^\star
    \end{aligned}
\end{equation*}
Minimizing the sum of the squared Frobenius errors of the cosine and sine coefficients is equivalent to approximating $\mathbf M_u$ by $\mathbf a_u\mathbf b_u^H$, whose rank is at most one.
By the Eckart-Young-Mirsky theorem~\citep{eckart1936approximation,mirsky1960symmetric}, the best rank-one approximation of $\mathbf M_u$ in Frobenius norm is given by its leading singular component.
We apply singular value decomposition (SVD) to $\mathbf M_u$ and retain its leading singular component, $\mathbf M_u\approx\sigma_u\mathbf l_u\mathbf r_u^H$.
Taking $\mathbf a_u=\sqrt{\sigma_u}\mathbf l_u$ and $\mathbf b_u=\sqrt{\sigma_u}\mathbf r_u$ gives
\begin{equation}
    \widetilde{\mathbf W}_Q^{(u)}
        =\sqrt \sigma_u
        \begin{bmatrix}
            (\operatorname{Re}\mathbf l_u)^\top\\
            (\operatorname{Im}\mathbf l_u)^\top
        \end{bmatrix}
        \Sigma_Q^{-1/2},
    \qquad
    \widetilde{\mathbf W}_K^{(u)}
        =\sqrt \sigma_u
        \begin{bmatrix}
            (\operatorname{Re}\mathbf r_u)^\top\\
            (\operatorname{Im}\mathbf r_u)^\top
        \end{bmatrix}
        \Sigma_K^{-1/2}.
    \label{eq:low-rank-structured-factor}
\end{equation}
Stacking the $U$ two-row blocks in Eq.~\eqref{eq:low-rank-structured-factor} gives the $2U$-dimensional projections required by Sec.~\ref{sec:circular-addresses}.
The frequency projection is optimal before imposing the plane structure, and each rank-one truncation is optimal for its coefficient matrix.
Because the basis functions can be correlated and unequally weighted under $p(\Delta)$, the separate truncations do not generally minimize the joint constrained objective in Eq.~\eqref{eq:low-rank-kernel-objective}.
They instead provide the structured initialization refined by forward KL alignment in Sec.~\ref{sec:migration-init}.

\noindent\textbf{NoPE initialization.}
For NoPE, $\mathbf R_S(\Delta)=\mathbf I$, so the target operator is independent of relative position. 
Under Eq.~\eqref{eq:low-rank-kernel-objective}, its optimal unconstrained value in the weighted space is $\widehat{\mathbf B}_0=\mathbb E_{\Delta\sim p}[\widehat{\mathbf B}_T(\Delta)]$: expanding the squared error around this mean leaves a constant term plus the squared distance to $\widehat{\mathbf B}_0$.
Since the target operator has rank at most $2U$, its optimal constrained approximation is the truncated SVD of this mean. 
Let $\mathbf U_r\mathbf S_r\mathbf V_r^\top$ denote the truncation with $r=\min(2U,d_k)$. The corresponding projections are
\begin{equation}
\widetilde{\mathbf W}_Q
=\mathbf S_r^{1/2}\mathbf U_r^\top\mathbf\Sigma_Q^{-1/2},
\qquad
\widetilde{\mathbf W}_K
=\mathbf S_r^{1/2}\mathbf V_r^\top\mathbf\Sigma_K^{-1/2}.
\label{eq:low-rank-nope-projections}
\end{equation}
If $r<2U$, zero rows complete the required dimensions. This construction averages the teacher's positional interactions under $p(\Delta)$ rather than simply removing its RoPE rotations.

Both constructions are compatible with the adapter in Sec.~\ref{sec:migration-init}. The weighted teacher operators have column space contained in that of $\mathbf\Sigma_Q^{1/2}\mathbf W_Q^\top$ and row space contained in that of $\mathbf W_K\mathbf\Sigma_K^{1/2}$. 
Frequency projection preserves these containments, as do the retained nonzero singular components. After inverse weighting, the constructed query and key projections therefore lie in the row spaces of $\mathbf W_Q$ and $\mathbf W_K$, respectively. Consequently, adapters realizing these projections are $\mathbf P_Q=\widetilde{\mathbf W}_Q\mathbf W_Q^\dagger$ and $\mathbf P_K=\widetilde{\mathbf W}_K\mathbf W_K^\dagger$, which satisfy $\widetilde{\mathbf W}_Q=\mathbf P_Q\mathbf W_Q$ and $\widetilde{\mathbf W}_K=\mathbf P_K\mathbf W_K$.

Finally, the derivation matches unscaled interactions. Including the factor $(2U/d_k)^{1/4}$ in both adapters preserves the teacher's dot-product coefficient because $(2U/d_k)^{1/2}/\sqrt{2U}=1/\sqrt{d_k}$. The resulting projections provide the initialization for forward KL alignment and subsequent circular address folding.

\begin{table*}[t]
  \centering
  \caption{Architectures of the teacher models.}
  \label{tab:teacher-architectures}
  \scriptsize
  \setlength{\tabcolsep}{4.0pt}
  \renewcommand{\arraystretch}{1.05}
  \begin{tabular}{@{}llrrrrrcc@{}}
    \toprule
    Family & Teacher
      & Layer & Hidden size & Head & KV Head & $d_{\mathrm{head}}$
      & Attn. & Native ctx. \\
    \midrule
    \multirow{3}{*}{FLA}
      & transformer-340M-10B  & 24 & 1024 & 32 & 32 & 32  & MHA & 4K \\
      & transformer-1.3B-100B & 24 & 2048 & 32 & 32 & 64  & MHA & 2K \\
      & transformer-2.7B-100B & 32 & 2560 & 20 & 20 & 128 & MHA & 2K \\
    \midrule
    \multirow{3}{*}{Qwen2.5}
      & 0.5B-Base & 24 & 896  & 14 & 2 & 64  & GQA & 32K \\
      & 1.5B-Base & 28 & 1536 & 12 & 2 & 128 & GQA & 32K \\
      & 3B-Base   & 36 & 2048 & 16 & 2 & 128 & GQA & 32K \\
    \midrule
    RedPajama & INCITE-Base-3B-v1 & 32 & 2560 & 32 & 32 & 80  & MHA & 2K \\
    \midrule
    Mistral & 7B-v0.1 & 32 & 4096 & 32 & 8  & 128 & GQA/SWA & 8K \\
    \midrule
    LLaMA & Llama-2-7b-hf & 32 & 4096 & 32 & 32 & 128 & MHA & 4K \\
    \bottomrule
  \end{tabular}
\end{table*}

\begin{table}[t]
\centering
\scriptsize
\setlength{\tabcolsep}{2.0pt}
\caption{Overview of Transformer-to-RAM-Net migration. The total training budget is 500M tokens per teacher.}
\label{tab:migration-overview}
\resizebox{\linewidth}{!}{
\begin{tabular}{l l l l r l l}
\toprule
Phase & Stage & Main Loss & Trainable Parameters & Tokens & Warmup & LR \\
\midrule
\multirow{2}{*}{Intermediate}
& Low-rank \& positional alignment
& Attention KL
& QK adapters, log-scales
& 7.209M
& 0.360M
& $3\times10^{-4}$ \\
& Circular soft quantization
& --
& --
& 0
& --
& -- \\
\midrule
\multirow{4}{*}{Recovery}
& Head-level
& Head-output NMSE
& $W_r$, $W_w$, scales, $\gamma_t$
& 49.152M
& 1.638M
& $1\times10^{-4}$ \\
& Attn.-block local
& Residual-update NMSE
& Attention branch
& 61.440M
& 2.048M
& $1\times10^{-4}$ \\
& Attn.-block stack
& Residual-update NMSE
& Attention branch
& 36.864M
& 1.638M
& $3\times10^{-5}$ \\
& Full-model
& $0.25$ Logit KL + CE
& All parameters
& 345.375M
& 0.819M
& $1\times10^{-5}$ \\
\bottomrule
\end{tabular}}
\end{table}

\section{Experiments}
\subsection{Teacher Models}
\label{app:teacher-model}

We evaluate the generality of our migration method on nine teacher models from five families with different architectural configurations, as summarized in Table~\ref{tab:teacher-architectures}. The five families are FLA~\cite{flahub_models}, Qwen2.5~\cite{qwen2025qwen25technicalreport}, RedPajama~\cite{redpajama_incite_3b}, Mistral~\cite{jiang2023mistral7b}, and LLaMA-2~\cite{touvron2023llama2openfoundation}.
All teachers are loaded from their corresponding publicly available pretrained checkpoints on Hugging Face. We use the 10BT sample of the FineWeb-Edu dataset~\cite{huggingfacefw_2024} for training. Their different head dimensions, numbers of query and KV heads, and attention configurations lead to different low-rank adapter shapes and head-wise CAPE/NoPE assignments during migration.

\begin{table*}[t]
  \centering
  \caption{Mass-aware readout results. Values in parentheses indicate the score differences relative to RAM-Net in Table~\ref{tab:main-eval-no-mass}.}
  \label{tab:main-eval-mass}
  \scriptsize
  \setlength{\tabcolsep}{1.8pt}
  \renewcommand{\arraystretch}{1.08}
  \resizebox{\textwidth}{!}{%
  \begin{tabular}{@{}llccccccccc@{}}
    \toprule
    Family & Teacher
      & Wiki PPL $\downarrow$
      & ARC-E $\uparrow$
      & ARC-C $\uparrow$
      & PIQA $\uparrow$
      & SciQ $\uparrow$
      & COPA $\uparrow$
      & Hella. $\uparrow$
      & Avg. $\uparrow$
      & R $\uparrow$ \\
    \midrule

    \multirow{3}{*}{FLA}
      & transformer-340M-10B
      & 35.28\,(+0.1) & \textbf{55.8\,(+1.1)} & \textbf{23.9\,(+0.2)} & 64.3\,(-0.4)
      & \textbf{80.3\,(+0.1)} & \textbf{66.0\,(+1.0)} & \textbf{31.1\,(+0.1)} & \textbf{53.6\,(+0.4)} & \textbf{87.7\,(+1.2)} \\
      & transformer-1.3B-100B
      & \textbf{21.83\,(-0.1)} & 52.3\,(-0.4) & 23.7\,(-1.1) & 67.7\,(-0.6)
      & 78.0\,(-0.8) & \textbf{72.0\,(+1.0)} & \textbf{35.8\,(+0.2)} & 54.9\,(-0.3) & 86.4\,(-0.1) \\
      & transformer-2.7B-100B
      & 20.09\,(+0.3) & \textbf{56.8\,(+0.3)} & \textbf{25.4\,(+0.8)} & 69.6\,(-0.7)
      & 79.4\,(-0.5) & \textbf{76.0\,(+2.0)} & 37.8\,(-0.5) & \textbf{57.5\,(+0.2)} & \textbf{89.6\,(+0.4)} \\

    \midrule
    \multirow{3}{*}{Qwen2.5}
      & 0.5B-Base
      & 23.55\,(+0.2) & \textbf{64.1\,(+0.5)} & \textbf{28.8\,(+2.4)} & 68.4\,(-0.1)
      & 81.9\,(0.0) & 69.0\,(-4.0) & 36.8\,(-0.1) & 58.2\,(-0.2) & \textbf{86.1\,(+6.9)} \\
      & 1.5B-Base
      & 17.30\,(+1.0) & 74.8\,(-0.5) & 40.2\,(-1.3) & \textbf{74.7\,(+0.4)}
      & 88.4\,(-1.4) & 77.0\,(-5.0) & 44.7\,(-0.9) & 66.6\,(-1.5) & 90.2\,(-4.8) \\
      & 3B-Base
      & 13.89\,(+0.3) & 77.2\,(0.0) & \textbf{44.2\,(+0.5)} & 77.0\,(0.0)
      & 91.6\,(-1.3) & \textbf{81.0\,(+2.0)} & 50.3\,(-0.4) & \textbf{70.2\,(+0.1)} & \textbf{92.3\,(+0.8)} \\

    \midrule
    RedPajama & INCITE-Base-3B-v1
      & 17.71\,(+1.1) & \textbf{63.9\,(+0.3)} & \textbf{29.8\,(+2.1)} & 71.4\,(-0.8)
      & \textbf{81.1\,(+0.5)} & 74.0\,(-3.0) & 41.0\,(-0.9) & 60.2\,(-0.3) & \textbf{82.1\,(+2.4)} \\
    \midrule
    Mistral & 7B-v0.1
      & 10.32\,(+0.2) & 76.7\,(-0.5) & \textbf{44.7\,(+0.6)} & 79.5\,(-0.2)
      & \textbf{92.3\,(+0.8)} & \textbf{88.0\,(+1.0)} & 57.4\,(-0.6) & \textbf{73.1\,(+0.2)} & \textbf{89.8\,(+0.4)} \\
    \midrule
    LLaMA & Llama-2-7b-hf
      & 12.98\,(+1.1) & 71.2\,(-0.4) & \textbf{39.4\,(+0.4)} & 76.2\,(0.0)
      & 89.0\,(-0.2) & 76.0\,(-5.0) & 50.4\,(-1.2) & 67.0\,(-1.1) & 84.6\,(-2.6) \\

    \bottomrule
  \end{tabular}%
  }
  \vspace{2pt}
\end{table*}

\subsection{Migration Details}
\label{sec:migration-add}

\paragraph{RAM-Net configurations.} Unless otherwise specified, RAM-Net uses $d_p=4$, $U=5$, and Top-$K=32$, and the retained address probabilities are not renormalized after Top-$K$. Output normalization is disabled throughout migration because the Transformer teacher has no corresponding operation. The sigmoid output gate is introduced only at the beginning of attention-block recovery. We initialize its projection weight to zero and bias to $b_g=3$, giving a constant gain $g_0=\sigma(3)$, and rescale the copied output projection as $W^o\leftarrow W^o/g_0$. This preserves the attention output at insertion while allowing the gate to become learnable in subsequent recovery.

\paragraph{Optimizer.} Table~\ref{tab:migration-overview} summarizes the complete training configuration. All stages use AdamW with zero weight decay and gradient clipping at 1. Forward and backward computation use BF16 with FP32 master weights. Each optimizer step processes 16K tokens. Low-rank alignment, head-level recovery, and attention-block recovery use linear warmup followed by cosine decay, while full-model recovery uses a short warmup followed by a constant learning rate. After each training stage, we restore the checkpoint with the lowest validation loss before proceeding to the next stage.

\paragraph{Overall workflow.} During low-rank and positional alignment, only the adapters and head-wise log-scales are optimized, while the teacher and backbone remain frozen. Circular soft quantization analytically folds the selected adapters into RAM-Net read and write projections and requires no training. Head-level recovery then trains the RAM routing parameters and dynamic forgetting parameters on the native Top-$K$ path while keeping the copied V/O projections and backbone frozen.

\begin{wrapfigure}{R}{0.4\textwidth}
  \centering
  \includegraphics[width=\linewidth]{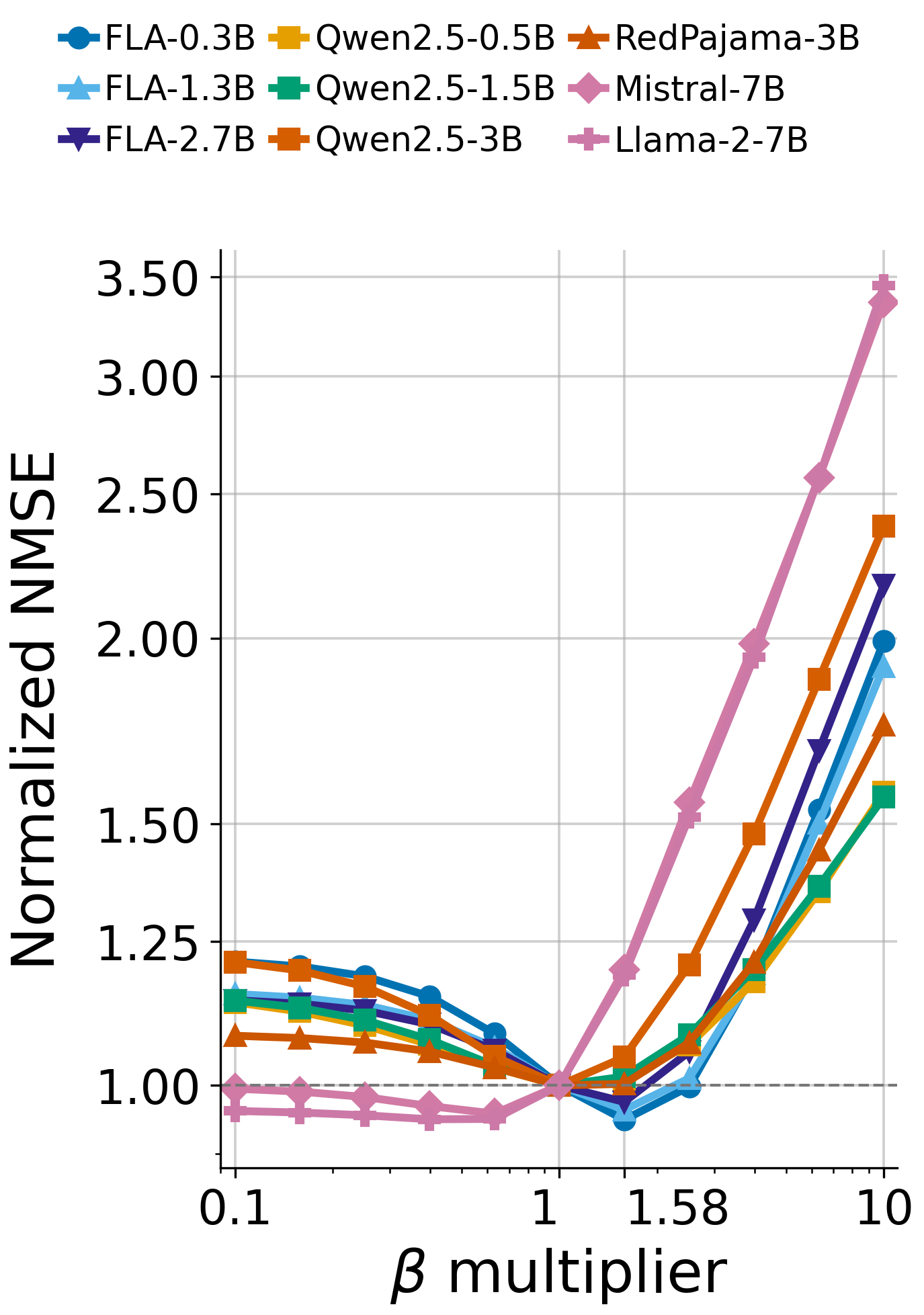}
  \caption{Head NMSE under different circular assignment concentrations $\beta$.}
  \label{fig:beta-nmse}
\end{wrapfigure}

Attention-block recovery contains two consecutive steps. The local step feeds each student block the corresponding teacher hidden state, while the stack step propagates student hidden states through the complete stack. Both optimize the residual-update objective $\mathcal L_{\mathrm{block}}=\frac{1}{L}\sum_{\ell=0}^{L-1}\operatorname{NMSE}\!\left(\Delta z_\ell^S,\Delta z_\ell^T\right)$, where $\Delta z_\ell=z_\ell-u_\ell$, and $u_\ell$ and $z_\ell$ denote the block input and output. Both steps train the complete attention branch, including RAM routing, dynamic forgetting, V/O projections, and the output gate, while MLP, normalization, embedding, and language-model-head parameters remain frozen. The copied V/O projections use $0.1\times$ the base learning rate.

Full-model recovery finally unfreezes all parameters and optimizes $\mathcal L_{\mathrm{full}}$. RAM-specific parameters use the base learning rate, while parameters copied from the teacher use $0.1\times$ the base learning rate.

For both attention-block recovery steps, copied V/O parameters use $0.1\times$ the listed learning rate. During full-model recovery, all parameters copied from the teacher similarly use $0.1\times$ the listed rate, while RAM-specific parameters use the full rate.

\subsection{Performance of Mass-Aware Readout}
\label{app:ramnet-plus}

Tables~\ref{tab:main-eval-mass} and~\ref{tab:mass-read-js} compare mass-aware and standard RAM-Net readouts. Consistent with the theoretical derivation in Sec.~\ref{sec:recurrent-readout}, mass-aware readout more closely matches the Transformer before head-level recovery, yielding lower attention JS in six of nine models. However, after migration, it generally gives higher Wiki PPL while maintaining comparable common-knowledge performance. This gap may be related to the coupling of read weights to forgetting dynamics and a potential mismatch between mass reweighting and Top-$K$ selection based only on $r_t(\mathbf a)$.

Table~\ref{tab:mass-tfs} further compares the two readouts when training a 340M model from scratch with the same FineWeb-Edu 10BT dataset. Mass-aware readout again gives slightly higher Wiki PPL, while improving the average common-knowledge score from 55.6 to 56.8. This trend is consistent with the migration results.

\begin{table*}[t]
  \centering
  \caption{JS divergence with and without mass-aware readout after enabling GSU.}
  \label{tab:mass-read-js}
  \scriptsize
  \setlength{\tabcolsep}{4.0pt}
  \renewcommand{\arraystretch}{1.0}
  \resizebox{\textwidth}{!}{%
  \begin{tabular}{@{}lccccccccc@{}}
    \toprule
    Readout rule
      & FLA-0.3B
      & FLA-1.3B
      & FLA-2.7B
      & Qwen-0.5B
      & Qwen-1.5B
      & Qwen-3B
      & Red.-3B
      & Mist.-7B
      & LLaMA-7B \\
    \midrule
    w/o
      & 0.277
      & 0.286
      & 0.308
      & 0.292
      & \textbf{0.293}
      & 0.312
      & \textbf{0.312}
      & 0.283
      & \textbf{0.253} \\
    w/
      & \textbf{0.259}
      & \textbf{0.273}
      & \textbf{0.300}
      & \textbf{0.290}
      & 0.297
      & \textbf{0.311}
      & 0.319
      & \textbf{0.280}
      & 0.257 \\
    \bottomrule
  \end{tabular}
  }
\end{table*}

\begin{table*}[t]
  \centering
  \caption{Train-from-scratch results for 340M models with and without mass-aware readout.}
  \label{tab:mass-tfs}
  \footnotesize
  \setlength{\tabcolsep}{3.5pt}
  \renewcommand{\arraystretch}{1.08}
  \begin{tabular}{@{}ccccccccc@{}}
    \toprule
      & Wiki PPL $\downarrow$
      & ARC-E $\uparrow$
      & ARC-C $\uparrow$
      & PIQA $\uparrow$
      & SciQ $\uparrow$
      & COPA $\uparrow$
      & Hella. $\uparrow$
      & Avg. $\uparrow$ \\
    \midrule
    w/o  & 26.89 & 59.1 & 24.7 & 66.9 & 82.5 & 67.0 & 33.4 & 55.6 \\
    w/ & 27.10 & 59.1 & 26.1 & 67.5 & 82.1 & 73.0 & 33.3 & 56.8 \\
    \bottomrule
  \end{tabular}

  \vspace{2pt}
\end{table*}

\subsection{Additional Experimental Results}
\label{app:exp-ab}

\paragraph{Detailed results on LongBench.}
Table~\ref{tab:longbench-full} reports code completion and summarization results in LongBench for all teacher and RAM-Net pairs. RAM-Net recovers teacher performance more closely for smaller models, while the gap generally becomes larger at higher model scales. Since all models use the same 500M-token migration budget, these results suggest that larger teachers may require a larger recovery budget to better transfer their downstream capabilities.

\paragraph{Head NMSE with concentration $\beta$.}
Fig.~\ref{fig:beta-nmse} shows a trend consistent with the attention JS results in Fig.~\ref{fig:beta-js}. Head NMSE is minimized near $\beta^\star$ across teachers, with some minima around $1.58\beta^\star$, further supporting $\beta^\star$ as a balanced default setting.

\paragraph{Head-recovery loss curves.}
Fig.~\ref{fig:nmse_app} extends the low-rank alignment and circular codebook comparison in Table~\ref{tab:address-initialization-ablation} to head-level recovery. Consistent with the lower attention JS reported in the main text, the circular codebook also starts from the lowest NMSE and shows the fastest error reduction during recovery. Together, these results show that the circular construction improves both the initial attention alignment and the efficiency of subsequent recovery.

\paragraph{Per-stage breakdown in migration.}
Using FLA-2.7B as an example, Fig.~\ref{fig:stage_error} shows how different migration stages affect attention matching and output recovery. Circular soft quantization introduces approximation error, while Top-$K$ truncation causes the largest increase in attention KL. This is expected because the address-overlap approximation in Sec.~\ref{sec:address-matching} is defined over the complete soft read and write distributions, whereas Top-$K$ discards part of the overlap and sparsifies the induced token-to-token matching. Despite this large attention mismatch, head-level recovery rapidly reduces the output NMSE while the attention KL remains high. This indicates that RAM-Net can recover teacher head outputs through its sparse address routing and recurrent updates without reproducing the teacher attention distributions exactly. Attention-block and full-model recovery further reduce the remaining output discrepancy. The slight increase in NMSE during full-model recovery reflects the shift from intermediate-output matching to the final logit KL and cross-entropy objectives.

\begin{table*}[t]
  \centering
  \caption{LongBench results for all teacher and RAM-Net pairs.}
  \label{tab:longbench-full}
  \scriptsize
  \setlength{\tabcolsep}{2.4pt}
  \renewcommand{\arraystretch}{1.05}
  \begin{tabular*}{\textwidth}{@{\extracolsep{\fill}}ll*{9}{r}@{}}
    \toprule
    Metric & Method
      & FLA-0.3B
      & FLA-1.3B
      & FLA-2.7B
      & Qwen-0.5B
      & Qwen-1.5B
      & Qwen-3B
      & Red.-3B
      & Mist.-7B
      & LLaMA-7B \\
    \midrule
    \multirow{2}{*}{LCC $\uparrow$}
      & Teacher & 17.0 & 48.3 & 50.7 & 10.0 & 14.7 & 14.1 & 44.7 & 70.0 & 66.7 \\
      & RAM-Net & 19.6 & 36.4 & 36.4 & 18.5 & 12.7 & 15.5 & 32.7 & 33.4 & 24.8 \\
    \cmidrule(lr){1-11}
    \multirow{2}{*}{MultiNews $\uparrow$}
      & Teacher & 10.5 & 6.6 & 6.3 & 21.4 & 25.2 & 24.5 & 20.1 & 15.5 & 13.8 \\
      & RAM-Net & 9.6 & 7.2 & 8.4 & 12.9 & 13.9 & 15.9 & 8.6 & 12.1 & 11.5 \\
    \cmidrule(lr){1-11}
    \multirow{2}{*}{SAMSum $\uparrow$}
      & Teacher & 6.6 & 6.2 & 4.9 & 37.3 & 42.4 & 45.5 & 0.0 & 28.2 & 31.4 \\
      & RAM-Net & 8.0 & 9.3 & 10.3 & 20.3 & 28.7 & 29.5 & 0.1 & 22.1 & 23.9 \\
    \cmidrule(lr){1-11}
    \multirow{2}{*}{Avg. $\uparrow$}
      & Teacher & 11.3 & 20.3 & 20.6 & 22.9 & 27.4 & 28.0 & 21.6 & 37.9 & 37.3 \\
      & RAM-Net & 12.4 & 17.6 & 18.3 & 17.3 & 18.4 & 20.3 & 13.8 & 22.5 & 20.1 \\
    \bottomrule
  \end{tabular*}

  \vspace{2pt}
\end{table*}

\paragraph{Different $d_p$ and $U$.} 
Fig.~\ref{fig:sweep_app} evaluates the effect of $d_p$ and $U$ during intermediate construction. Increasing $U$ generally reduces the forward KL and head NMSE by preserving a higher-dimensional QK representation, while a larger $d_p$ provides finer angular resolution for the circular codebook. However, lower intermediate error does not necessarily translate into better final migration performance. Since the number of RAM-Net slots grows as $d_p^U$, increasing either parameter also enlarges the address space and can reduce overlap mass and effective slot utilization, as observed in the original RAM-Net study. Table~\ref{tab:order-harness} reflects this trade-off, where $U=5$ achieves the best final Avg. and relative score despite $U=6$ having lower low-rank KL. We therefore use $d_p=4$ and $U=5$ as a balanced setting between intermediate approximation quality and final language-model performance.

\begin{figure}[t]
\centering
\includegraphics[width=\linewidth]{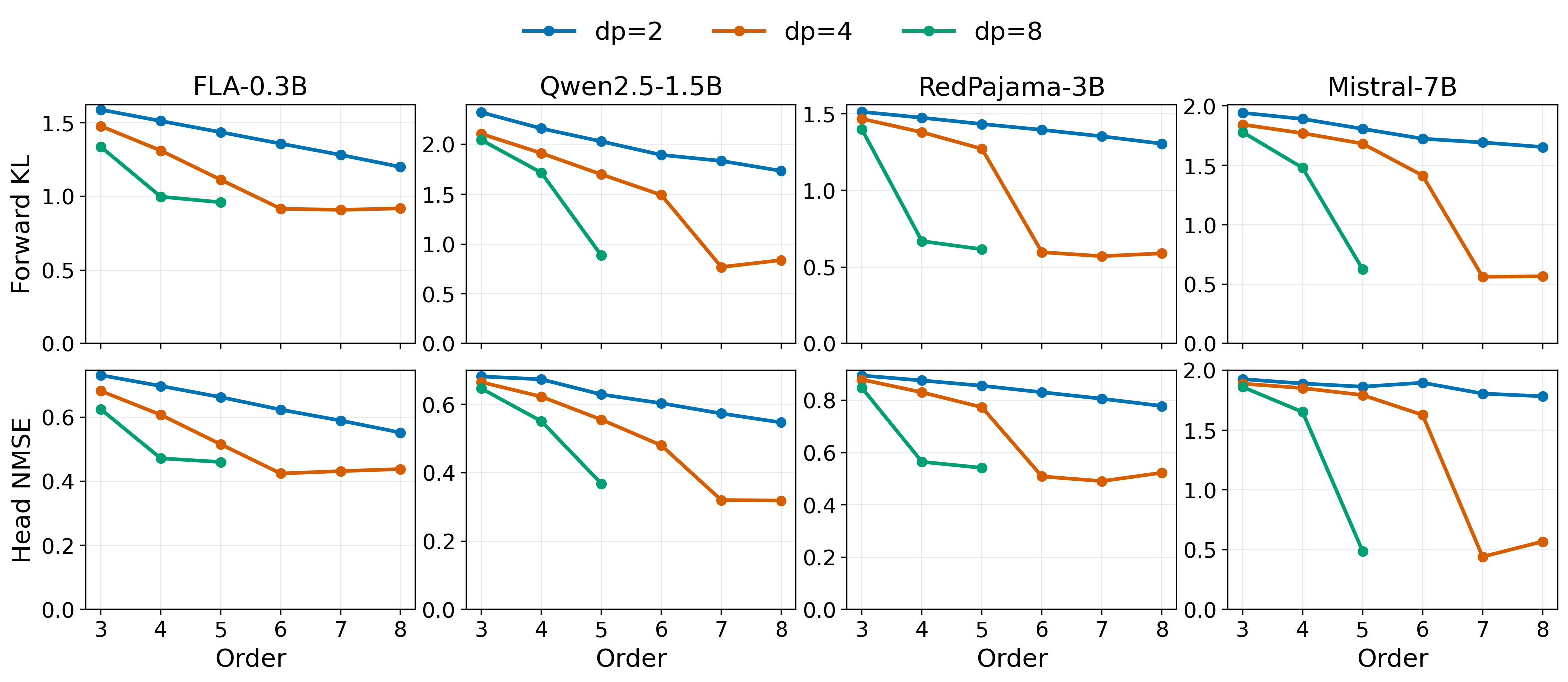}
\caption{Effect of $d_p$ and order $U$ on intermediate forward KL and head NMSE.}
\label{fig:sweep_app}
\end{figure}

\begin{table*}[t]
  \centering
  \caption{Effect of order $U$ on Qwen2.5-1.5B. Logit KL is measured after full-model recovery.}
  \label{tab:order-harness}
  \scriptsize
  \setlength{\tabcolsep}{1.4pt}
  \renewcommand{\arraystretch}{1.05}
  \begin{tabular*}{\textwidth}{@{\extracolsep{\fill}}c*{13}{c}@{}}
    \toprule
    \begin{tabular}[c]{@{}c@{}}Order\end{tabular}
      & \begin{tabular}[c]{@{}c@{}}Low-rank\\KL $\downarrow$\end{tabular}
      & \begin{tabular}[c]{@{}c@{}}Quant.\\JS $\downarrow$\end{tabular}
      & \begin{tabular}[c]{@{}c@{}}Head\\NMSE $\downarrow$\end{tabular}
      & \begin{tabular}[c]{@{}c@{}}Logit\\KL $\downarrow$\end{tabular}
      & \begin{tabular}[c]{@{}c@{}}Wiki\\PPL $\downarrow$\end{tabular}
      & \begin{tabular}[c]{@{}c@{}}ARC-E $\uparrow$\end{tabular}
      & \begin{tabular}[c]{@{}c@{}}ARC-C $\uparrow$\end{tabular}
      & \begin{tabular}[c]{@{}c@{}}PIQA $\uparrow$\end{tabular}
      & \begin{tabular}[c]{@{}c@{}}SciQ $\uparrow$\end{tabular}
      & \begin{tabular}[c]{@{}c@{}}COPA $\uparrow$\end{tabular}
      & \begin{tabular}[c]{@{}c@{}}Hella. $\uparrow$\end{tabular}
      & \begin{tabular}[c]{@{}c@{}}Avg. $\uparrow$\end{tabular}
      & \begin{tabular}[c]{@{}c@{}}R $\uparrow$\end{tabular} \\
    \midrule
    4 (256 slots)  & 1.957 & \textbf{0.253} & \textbf{0.177} & 0.500 & 16.58 & 73.8 & 38.9 & 73.6 & 87.4 & \textbf{82.0} & 45.6 & 66.9 & 90.8 \\
    5 (1024 slots) & 1.748 & 0.293 & 0.197 & \textbf{0.465} & 16.33 & \textbf{75.3} & 41.5 & \textbf{74.3} & \textbf{89.8} & \textbf{82.0} & \textbf{45.6} & \textbf{68.1} & \textbf{95.0} \\
    6 (4096 slots) & \textbf{1.555} & 0.339 & 0.209 & 0.476 & \textbf{16.01} & 74.4 & \textbf{41.6} & 74.0 & 89.2 & 77.0 & 45.4 & 66.9 & 91.8 \\
    \bottomrule
  \end{tabular*}

  \vspace{2pt}
\end{table*}

\begin{table*}[t]
  \centering
  \caption{Ablation of the staged migration pipeline on Qwen2.5-1.5B.}
  \label{tab:stage-ablation-details}
  \scriptsize
  \setlength{\tabcolsep}{2pt}
  \renewcommand{\arraystretch}{1.08}
  \begin{tabular}{@{}llrrrrrrrrrrr@{}}
    \toprule
    Method & Checkpoint & Tokens (M)
      & Logit KL $\downarrow$ & Wiki $\downarrow$
      & ARC-E & ARC-C & PIQA & SciQ & COPA & Hella.
      & Avg. $\uparrow$ & R $\uparrow$ \\
    \midrule
    Teacher & Original & -- & -- & 12.60
      & 75.2 & 41.1 & 75.6 & 94.2 & 83.0 & 50.2 & 69.9 & -- \\
    \midrule
    Complete & Head-Rec. end & 56.4 & -- & 77.56
      & 47.1 & 25.1 & 64.2 & 69.9 & 69.0 & 35.2 & 51.7 & 43.8 \\
         & Block-Rec. end & 154.7 & -- & 19.44
      & 72.1 & 37.6 & 73.6 & 87.2 & 79.0 & 44.1 & 65.6 & 86.3 \\
         & Final & 500.0 & 0.164 & 16.33
      & 75.3 & 41.5 & 74.3 & 89.8
      & 82.0 & 45.6 & 68.1 & 95.0 \\
    \midrule
    Full-Model only & Matched head & 56.4 & 0.419 & 23.90
      & 65.1 & 31.5 & 70.2 & 83.0 & 77.0 & 40.7 & 61.2 & 71.1 \\
              & Matched block & 154.7 & 0.278 & 19.35
      & 70.1 & 35.7 & 73.6 & 87.3 & 78.0 & 43.2 & 64.7 & 82.6 \\
              & Final & 500.0 & 0.197 & 16.82
      & 71.8 & 39.2 & 74.3 & 89.5 & 80.0 & 44.6 & 66.6 & 89.7 \\
    \midrule
    Pattern-free CAPE & Final & 500.0 & 0.162 & 16.17
      & 74.3 & 41.1 & 74.3 & 89.1 & 81.0 & 45.8 & 67.6 & 93.7 \\
    w/o head Rec. & Final & 500.0 & 0.166 & 16.31
      & 72.8 & 39.7 & 74.4 & 89.2 & 80.0 & 45.2 & 66.9 & 90.8 \\
    w/o block Rec. & Final & 500.0 & 0.181 & 16.64
      & 74.3 & 39.2 & 73.9 & 89.4 & 79.0 & 45.4 & 66.9 & 90.2 \\
    \bottomrule
  \end{tabular}
\end{table*}

\paragraph{Ablations on staged migration.}
Table~\ref{tab:stage-ablation-details} provides a detailed comparison of the migration pipeline. \emph{Complete pipeline} denotes the full staged procedure, while \emph{full-model only} performs only full-model recovery. The ``matched'' checkpoints use the same cumulative token budgets as the corresponding stages of the complete pipeline. \emph{Pattern-free CAPE selection} determines CAPE heads directly from their relative alignment errors without the preset pattern. \emph{w/o head-level recovery} and \emph{w/o attention-block recovery} omit the corresponding recovery stage and reallocate its training tokens to full-model recovery.

The \emph{full-model only} baseline achieves higher task scores at the early matched-head budget, but the \emph{complete pipeline} overtakes it after attention-block recovery and reaches higher final Avg. and relative score $R$ under the same 500M-token budget. Removing either head-level or attention-block recovery also lowers the final Avg. and $R$, even when Wiki PPL remains similar. These results show that the intermediate recovery stages improve teacher-relative downstream performance beyond what can be achieved by allocating the same training budget directly to full-model recovery.

\begin{figure}[t]
\centering
\includegraphics[width=\linewidth]{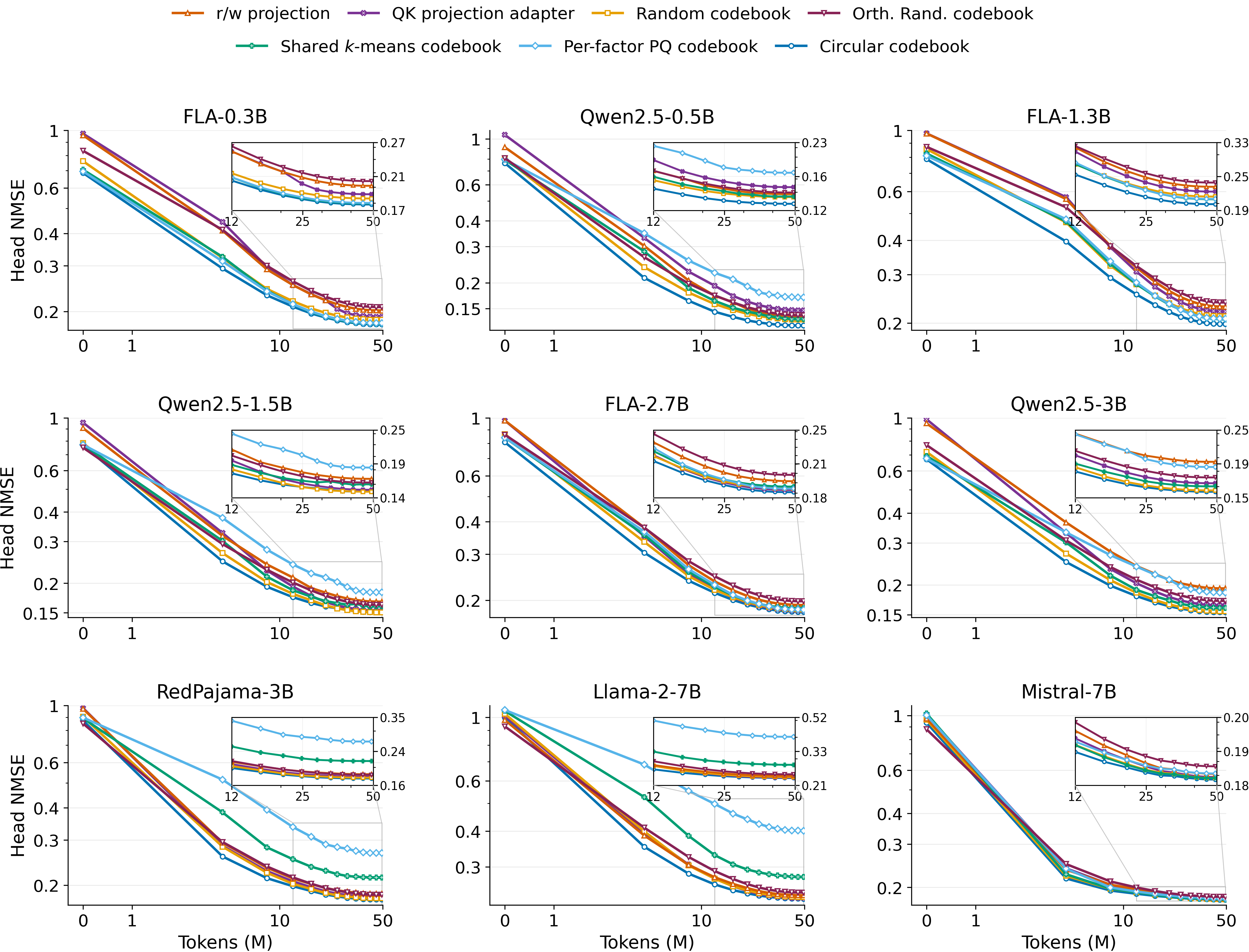}
\caption{Head-level recovery NMSE under different low-rank adapter and codebook constructions.}
\label{fig:nmse_app}
\end{figure}

\begin{figure}[t]
\centering
\includegraphics[width=\linewidth]{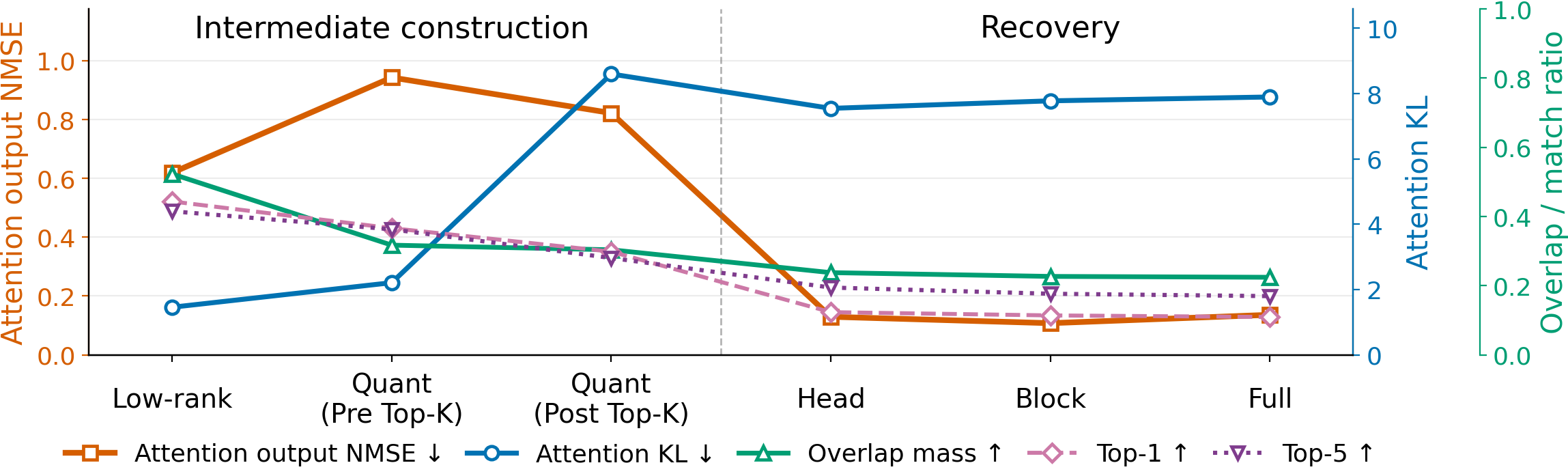}
\caption{Stage-wise evolution throughout the migration pipeline of FLA-2.7B.}
\label{fig:stage_error}
\end{figure}